\documentclass[conference]{IEEEtran}
\usepackage{times}

\usepackage[numbers, sort&compress]{natbib}
\usepackage[table,xcdraw]{xcolor}
\usepackage{amsmath,amsfonts,amssymb,amsthm}
\usepackage[ruled, lined, linesnumbered, commentsnumbered, longend]{algorithm2e}
\usepackage{bbm}
\usepackage{bm}
\usepackage{multirow}
\usepackage{graphicx}
\usepackage{subcaption}
\usepackage{float}
\usepackage{siunitx}
\usepackage{relsize}
\usepackage{mdframed}
\usepackage{xspace}
\usepackage{url}
\usepackage{boldline}
\usepackage{stfloats}
\let\labelindent\relax
\usepackage{enumitem}
\usepackage{pifont}
\usepackage{nicefrac}
\usepackage{booktabs}
\usepackage[absolute,overlay]{textpos}
\usepackage{fancyvrb}
\DefineVerbatimEnvironment{codesnippet}{Verbatim}{
  fontsize=\scriptsize
}

\renewcommand{\IEEEQED}{\IEEEQEDopen}

\newcommand{\mycroppedimg}[1]{%
  \includegraphics[
    width=\linewidth,
    trim=15cm 5cm 15cm 5cm, % L B R T
    clip
  ]{#1}%
}

\DeclareMathOperator*{\argmax}{arg\,max}

\DeclareMathOperator{\tr}{tr}

\newtheorem{problem}{Problem}

\newtheorem{assumption}{Assumption}
\newtheorem{theorem}{Theorem}
\newtheorem{lemma}{Lemma}
\newtheorem{corollary}{Corollary}
\newtheorem{proposition}{Proposition}
\newtheorem{definition}{Definition}
\newtheorem{remark}{Remark}
\definecolor{MaterialBlueGray10}{HTML}{CFD8DC}
\definecolor{MaterialBlueGray900}{HTML}{263238}
\definecolor{MaterialBlue10}{HTML}{E3F2FD}
\definecolor{MaterialBlue900}{HTML}{0D47A1}
\newtheorem{running-example}{Running Example}
\newtheorem*{running-example*}{Running Example}
\newcommand{\RNum}[1]{\uppercase\expandafter{\romannumeral #1\relax}}
\newcommand\Mark[1]{\textsuperscript#1}

\usepackage[bookmarks=true]{hyperref}
\hypersetup{
    colorlinks,
    linkcolor={red!50!black},
    citecolor={blue!50!black},
    urlcolor={blue!80!black}
}

\begin{document}

\begin{textblock}{9}(1.412, 0.343)
{
\sffamily
\noindent\textbf{\textsf{Robotics: \hspace{.09ex}Science and Systems 2026}} \\
\noindent\textbf{\textsf{Sydney, Australia, July 13-July 17, 2026}}
}
\end{textblock}

\title{Learning What Matters: Adaptive Information-Theoretic Objectives for Robot Exploration}

\author{
\IEEEauthorblockN{Youwei Yu\Mark{1},
Jionghao Wang\Mark{2},
Zhengming Yu\Mark{2},
Wenping Wang\Mark{2} and
Lantao Liu\Mark{1}}
\IEEEauthorblockA{\Mark{1}Luddy School of Informatics, Computing, and Engineering,
Indiana University, Bloomington, IN, USA\\
Email: \texttt{\{youwyu, lantao\}@iu.edu}}
\IEEEauthorblockA{\Mark{2}Department of Computer Science \& Engineering,
Texas A\&M University, College Station, TX, USA\\
Email: \texttt{\{jionghao, yuzhengming, wenping\}@tamu.edu}}
}

\IEEEoverridecommandlockouts

\maketitle

\def\thefootnote{}\footnotetext{Website with additional information, videos, and code: \url{https://www.youwei-yu.com/qoed}.}
\def\thefootnote{\arabic{footnote}}

\begin{abstract}
Designing learnable information-theoretic objectives for robot exploration remains challenging. Such objectives aim to guide exploration toward data that reduces uncertainty in model parameters, yet it is often unclear what information the collected data can actually reveal. Although reinforcement learning (RL) can optimize a given objective, constructing objectives that reflect parametric learnability is difficult in high-dimensional robotic systems. Many parameter directions are weakly observable or unidentifiable, and even when identifiable directions are selected, omitted directions can still influence exploration and distort information measures.
To address this challenge, we propose Quasi-Optimal Experimental Design (Q{\footnotesize OED}), an adaptive information objective grounded in optimal experimental design. Q{\footnotesize OED} (i) performs eigenspace analysis of the Fisher information matrix to identify an observable subspace and select identifiable parameter directions, and (ii) modifies the exploration objective to emphasize these directions while suppressing nuisance effects from non-critical parameters. Under bounded nuisance influence and limited coupling between critical and nuisance directions, Q{\footnotesize OED} provides a constant-factor approximation to the ideal information objective that explores all parameters.
We evaluate Q{\footnotesize OED} on simulated and real-world navigation and manipulation tasks, where identifiable-direction selection and nuisance suppression yield performance improvements of \SI{35.23}{\percent} and \SI{21.98}{\percent}, respectively. When integrated as an exploration objective in model-based policy optimization, Q{\footnotesize OED} further improves policy performance over established RL baselines.
\end{abstract}

\IEEEpeerreviewmaketitle

\section{Introduction}
In robot exploration, information-theoretic objectives provide a principled mechanism for making effective use of limited interaction budgets by encouraging actions that collect data reducing uncertainty about unknown quantities.
Bayesian optimal experimental design (BOED)~\cite{rainforth2024modern} formalizes this idea by selecting
actions that maximize expected information gain about a chosen set of parameters~\cite{rainforth2024modern}. This perspective has been used across robotics,
including active perception for scene representations~\cite{jiang2023fisherrf,XieY-RSS-25,Strong25nbv}, sim-to-real transfer via simulator calibration~\cite{memmel2024asid,sobanbabu2025sampling}, and manipulation via contact-rich interactions~\cite{SathyanarayanH-RSS-25}. Due to its statistical rigor, BOED can substantially improve exploration efficiency by prioritizing informative data.

However, the success of BOED-based exploration depends on a well-specified objective---in particular, choosing parameters that the collected data
can reliably reveal (i.e., that are sufficiently observable and identifiable). In practice, learnability varies widely across settings: geometry can
be more observable than appearance in active perception~\cite{jiang2023fisherrf,XieY-RSS-25,Strong25nbv}, and mass/motor properties can be easier to infer than aerodynamic effects in system identification~\cite{sobanbabu2025sampling}. In challenging scenarios, such as a quadruped traversing ice, it may be difficult to determine whether abnormal behavior is caused by the environment or by degraded actuators~\cite{Ewen-RSS-24}. When the objective is well defined, BOED can also yield interpretable behaviors (e.g., rubbing to estimate friction)~\cite{SathyanarayanH-RSS-25}. Nonetheless, specifying the ``critical'' parameters typically requires domain expertise and often assumes they can be pre-specified \emph{a priori}. Moreover, naively ignoring unselected parameters can distort information objectives and lead exploration to pursue spurious or uninformative directions.

This work addresses the challenge of adaptively designing learnable information-theoretic objectives.
We focus on identifying critical parameters online and reducing the influence of nuisance parameters on the exploration objective.
To this end, we propose Quasi-Optimal Experimental Design (Q{\footnotesize OED}).
Q{\footnotesize OED} first leverages eigenspace analysis of the Fisher information matrix to identify an observable subspace and select identifiable parameter
coordinates. It then constructs an adaptive information objective that emphasizes the critical coordinates while suppressing nuisance directions.

Our main contributions include:
\begin{enumerate}
    \item \textbf{Identifying Critical Parameters.} 
    Many physical parameters in robotic systems are weakly observable or unidentifiable. We analyze the eigenstructure of the FIM to characterize an observable subspace and select a compact set of identifiable parameter directions.
    \item \textbf{Adaptive Information Objective.} 
    The proposed adaptive information-theoretic objective emphasizes critical directions of interest while suppressing the influence of nuisance parameters. Under bounded nuisance effects and limited coupling between critical and nuisance directions, this objective yields a constant-factor approximation to the ideal objective that explores all parameters.
    \item \textbf{Q{\footnotesize OED} for Model-Based Policy Optimization.}
    Beyond pure exploration, we integrate Q{\footnotesize OED} into model-based policy optimization (MBPO). We learn a physics-conditioned forward dynamics model that enables estimation of the Q{\footnotesize OED} objective, and augment the task reward with this objective to guide robot exploration.
\end{enumerate}

\section{Related Work}
Exploration is a central challenge in robotics because real-world interaction is expensive. In this section, we review (i) exploration methods in RL, (ii) BOED for exploration objectives, and (iii) model-based policy optimization, highlighting the practical gaps that arise when bringing BOED into real-world robotic learning.

\subsection{Learning to Explore}
Exploration in robotics has moved from heuristics-driven~\cite{cao2023representation} to learning-based methods.
In RL, exploration is typically implemented through additional reward bonuses and is often grouped into two families:
\emph{uncertainty-driven} (information gathering) and \emph{intrinsic-motivation} (curiosity).

\textbf{Uncertainty-driven exploration.}
It encourages actions that reduce uncertainty.
A common distinction is between:
(i) uncertainty about learned parameters (e.g., value functions), and
(ii) irreducible randomness in the environment (e.g., stochastic transitions).
In discrete settings, the uncertainty can be propagated with Bayesian updates~\cite{dearden1998bayesian}.
In continuous domains, it is often approximated using low-dimensional structure~\cite{azizzadenesheli2018efficient} or ensembles/bootstrapping~\cite{pmlr-v48-osband16,JMLR:v20:18-339}.
Exploration is then driven by reducing parameter uncertainty~\cite{dearden1998bayesian}, controlling information regret~\cite{pmlr-v75-kirschner18a}, or reducing uncertainty in predicted returns~\cite{moerland2017efficient}.
Irreducible randomness is commonly modeled with distributional RL~\cite{bellemare2017distributional,mavrin2019distributional}.
Although theoretically distinct, these uncertainty sources often overlap in practice~\cite{moerland2017efficient,nikolov2018informationdirected}.

\textbf{Curiosity-driven exploration.}
Curiosity bonuses encourage agents to seek experiences that are either novel or hard to predict.
For novelty, classic count-based methods in tabular RL~\cite{NIPS2017_3a20f62a} have been extended to continuous spaces using density estimation~\cite{NIPS2016_afda3322,choshen2018dora}.
Go-Explore~\cite{ecoffet2021first} further biases exploration by returning to previously discovered states and expanding from them.
For prediction-based curiosity, Random Network Distillation (RND)~\cite{burda2018exploration} uses the prediction error of a fixed random target network as an intrinsic reward.
Related methods use forward dynamics prediction~\cite{stadie2015incentivizing,4141061} or inverse dynamics to emphasize controllable novelty~\cite{pathakICMl17curiosity}.
Physics-based curiosity can also be defined through parameter estimation error~\cite{denil2016learning}, but it typically assumes a small and specified set of parameters.

\subsection{Bayesian Optimal Experiment Design}
BOED selects experiments (e.g., actions) to maximize expected information gain about unknown parameters. However, classical BOED typically relies on two assumptions: (i) a compact, expert-defined set of parameters, and (ii) an accurate probabilistic model for predicting information gain.

\textbf{Choosing parameters of interest.}
BOED traditionally assumes key parameters are pre-specified by experts.
With a good choice of parameters, BOED has achieved strong results in locomotion~\cite{sobanbabu2025sampling}, underactuated control~\cite{6882246}, and manipulation~\cite{oliveira2024bayesian,SathyanarayanH-RSS-25}.
However, success hinges on parameters that are both observable and identifiable (e.g., masses and damping in a cart--pendulum system~\cite{6882246}). This assumes that unselected parameters do not alter experiment utility. While active subspaces~\cite{ActiveSubspaces} provide tools for dimensionality reduction, they are typically applied to static experimental design. Standard Bayesian active learning can fail in the presence of nuisance parameters~\cite{sloman2024bayesian}. Nonetheless, identifying nuisance parameters and accounting for their influence remains a significant gap.

\textbf{Statistical model.} 
In policy learning, the statistical model usually represents forward dynamics~\cite{sobanbabu2025sampling}. Robotics applications often employ manually-designed physics models with Gaussian noise to ensure analytical tractability~\cite{memmel2024asid,sobanbabu2025sampling,SathyanarayanH-RSS-25}. While effective in controlled settings, these are sensitive to model mismatch. Conversely, model-based reinforcement learning (MBRL) excels at learning dynamics~\cite{hansen2024tdmpc} but often lacks the tractable likelihoods required for BOED. Consequently, a practical gap persists between predicting the next state and estimating the information gain it yields.

\subsection{Model-Based Policy Optimization}
Much recent progress in robotic RL relies on models that can generate large amounts of training experience with little additional data collection.
These models are typically either (i) physics simulators or (ii) learned dynamics models.

\textbf{Simulators.}
Physics simulators act as a library of validated physical rules and have enabled impressive sim-to-real transfer in agile locomotion and manipulation~\cite{David24anymalparkour,zhuang2024humanoid,Fabian24DTC,Handa23DeXtreme,Villasevil-RSS-24}.
Ongoing work continues to improve simulator fidelity and breadth~\cite{newton}, but accurately modeling every aspect of the real world (e.g., fluids, contact micro-effects) remains difficult, and the sim-to-real gap can still limit performance.

\textbf{Learned world models.}
MBRL learns a dynamics model directly from real experience and uses it to generate imagined rollouts in parallel~\cite{janner2019mbpo,li2025robotic}.
Training policies using imagined rollouts from a learned model is often referred to as model-based policy optimization~\cite{janner2019mbpo}.
Our proposed approach is based on this framework, but we emphasize an additional requirement that is critical for BOED:
the model should support likelihood-based reasoning about unknown parameters.

In contrast to existing work, our model is designed to be both a forward dynamics model for imagined rollouts and a statistical model that supports BOED-driven exploration.

% ===========================================================================================
\section{Preliminary}
We study a hidden-parameter Markov decision process (MDP) defined by the tuple $\langle \mathcal{S}, \mathcal{A}, \Phi, p_0, P, p(\bm{\phi}), R, T, \gamma \rangle$. States are $\mathbf{s}\in\mathcal{S}\subseteq\mathbb{R}^{d_s}$, actions are $\bm{a}\in\mathcal{A}\subseteq\mathbb{R}^{d_a}$, and the hidden parameters are $\bm{\phi}\in\Phi\subseteq\mathbb{R}^{m}$ with prior $p(\bm{\phi})$. The initial state is $\mathbf{s}_0\sim p_0(\mathbf{s}_0)$. Given $(\mathbf{s}_t,\bm{a}_t,\bm{\phi})$, the next state is sampled from the transition model $P(\cdot\vert \mathbf{s}_t,\bm{a}_t,\bm{\phi})$, i.e.,
$\mathbf{s}_{t+1}\sim p(\mathbf{s}_{t+1}\vert \mathbf{s}_t,\bm{a}_t,\bm{\phi})$.
The per-step reward is $R:\mathcal{S}\times\mathcal{A}\to\mathbb{R}$, the horizon is $T\in\mathbb{N}^+$, and $\gamma\in(0,1]$ is the discount factor. We write a trajectory prefix as
$ \bm{\tau}_t = [\mathbf{s}_0,\bm{a}_0,\mathbf{s}_1,\dots,\bm{a}_{t-1},\mathbf{s}_t] $. For a stochastic policy $\bm{\pi}(\bm{a}\vert \mathbf{s})$, the induced distribution over $\bm{\tau}_t$ is
\begin{equation}
\label{eq:traj-dist}
p(\bm{\tau}_t\vert \bm{\phi},\bm{\pi})
=
p_0(\mathbf{s}_0)\prod_{k=0}^{t-1}
\bm{\pi}(\bm{a}_k\vert \mathbf{s}_k)\,
p(\mathbf{s}_{k+1}\vert \mathbf{s}_k,\bm{a}_k,\bm{\phi}).
\end{equation}

\begin{mdframed}[hidealllines=true,backgroundcolor=gray!8,innerleftmargin=.2cm,
  innerrightmargin=.2cm]
\begin{problem}[Policy learning with an exploration objective]
\label{problem:mpc}
We learn a stochastic policy $\bm{\pi}(\bm{a}\vert \mathbf{s})$ that maximizes reward while also collecting data that helps estimate
the hidden parameters $\bm{\phi}$. Concretely, we solve
\begin{equation}
\label{eqn:raw-mdp}
\begin{aligned}
\bm{\pi}^{\star}
&=
\argmax_{\bm{\pi}\in\Pi}\;
\mathbb{E}_{ p(\bm{\phi}),\; p(\bm{\tau}\vert \bm{\phi},\bm{\pi})}
\Bigg[
\\
&\qquad
\sum_{t=0}^{T-1}\gamma^t
\Big(
R(\mathbf{s}_t,\bm{a}_t)
+\alpha\,\mathcal{B}(\bm{\phi} \vert \bm{\tau}_t)
\Big)
\Bigg]
\\
\text{s.t.}\quad
&\mathbf{s}_{t+1}\sim
p(\mathbf{s}_{t+1}\vert \mathbf{s}_t,\bm{a}_t,\bm{\phi}),
\mathbf{s}_t\in\mathcal{S},
\quad
\bm{a}_t\in\mathcal{A},
\quad
\forall t.
\end{aligned}
\end{equation}
where $\mathcal{B}(\bm{\phi} \vert \bm{\tau}_t)$ is an exploration objective computed from the trajectory prefix $\bm{\tau}_t$, and $\alpha\ge 0$ controls the trade-off between task reward and exploration.
\end{problem}
\end{mdframed}

Across the literature, $\bm{\phi}$ can represent different unknowns, ranging from explicit physical parameters~\cite{memmel2024asid}
to weights of neural networks~\cite{eschenhagen2023kroneckerfactored}. In this paper, $\bm{\phi}$ denotes physical parameters (e.g., friction). These parameters are interpretable and can be used as inputs to a forward dynamics model. We focus on physically meaningful parameters to avoid the degeneracy caused by arbitrary unit scalings.

\subsection{Exploration Objective from Information Gain}
\label{sec:prelimilary-info-gain}
We design the exploration objective $\mathcal{B}_t$ to encourage collecting data that is informative about $\bm{\phi}$.
Intuitively, data is informative if changing $\bm{\phi}$ would noticeably change the likelihood of the observed transitions.
We motivate this using Bayesian optimal experiment design (BOED) and the Fisher information matrix (FIM).

\begin{mdframed}[hidealllines=true,backgroundcolor=MaterialBlueGray10!40,innerleftmargin=.2cm,
  innerrightmargin=.2cm]
\begin{definition}[BOED with Fisher Information]
\label{def:boed-fim}
BOED chooses a design to maximize the expected information gain (EIG) about unknown parameters. In our setting, the ``design'' is the policy $\bm{\pi}$. We quantify the information content of a trajectory $\bm{\tau}_t$ with respect to $\bm{\phi}$ using FIM, which measures the sensitivity of the trajectory distribution to parameter perturbations. For a trajectory prefix $\bm{\tau}_t$, define the score
$\mathbf{g}(\bm{\tau}_t,\bm{\phi}) := \nabla_{\bm{\phi}} \log p(\bm{\tau}_t \vert \bm{\phi},\bm{\pi})$. The FIM is
\begin{equation}
\label{eqn:fim}
\bm{\mathcal{F}}_{\bm{\phi}}
=
\mathbb{E}_{\bm{\tau}_t \sim p(\bm{\tau}_t \vert \bm{\phi},\bm{\pi})}
\Big[
\mathbf{g}(\bm{\tau}_t,\bm{\phi})\;
\mathbf{g}(\bm{\tau}_t,\bm{\phi})^{\top}
\Big] ,
\end{equation}
assuming standard regularity conditions (Sect.~2.3.1 in~\cite{schervish2012theory}). Since the policy $\bm{\pi}(\bm{a}|\bm{s})$ and initial state distribution $p(\bm{s}_0)$ are independent of the parameters $\bm{\phi}$, their gradients vanish. The score depends solely on the transition dynamics:
\begin{equation}
\label{eq:qom-fim}
\nabla_{\bm{\phi}} \log p(\bm{\tau}_t \vert \bm{\phi},\bm{\pi})
=
\sum_{k=0}^{t-1}
\nabla_{\bm{\phi}}
\log p(\mathbf{s}_{k+1} \vert \mathbf{s}_k, \bm{a}_k, \bm{\phi}) .
\end{equation}
The full derivation is given in Appendix~\ref{apdx:deri-traj-log-likelihood}.
\end{definition}
\end{mdframed}

To obtain a single scalar measure of informativeness, we summarize the FIM using its trace (i.e., T-optimality~\cite{optimal-design-exps}):
\begin{equation}
\label{eq:trace-bonus}
\mathrm{Ideal\ BOED\ Objective:}\quad
\mathcal{B}_{\text{BOED}}(\bm{\phi} \vert \bm{\tau}_t)
=
\operatorname{tr}\!\left(\bm{\mathcal{F}}_{\bm{\phi}}\right).
\end{equation}

\subsection{Challenges to Address}
\label{subsec:challenges}

The success of BOED-style exploration depends on ideally an expert-curated, low-dimensional set of critical parameters. This expert curation relies on the implicit assumption that other excluded parameters exert negligible influence on the policy exploration. Real robots often provide neither. This leads to two questions that we address in this paper:
\begin{enumerate}   
    \item[\textbf{Q1}]\label{preliminary:q2} \textbf{Critical parameter subspace.}
    In high-dimensional systems, only a few ``critical'' parameters affect what the information can reveal~\cite{karakida2019universal,WIELAND202160} (see Fig.~\ref{fig:mujoco-fisher-info} for an illustration). Can an agent identify these critical parameters automatically, without expert assumptions?
    
    \item[\textbf{Q2}]\label{preliminary:q3} \textbf{Subspace-optimal experimental design.}
    Simply optimizing for critical parameters is insufficient; discarded parameters inevitably affect the information objective. Consequently, how can we design an objective to compensate for these unmodeled effects, ensuring the policy remains near optimal to the full BOED---the ideal objective computed over the entire parameter space?
\end{enumerate}

\begin{figure}[t!]
    \centering
    \includegraphics[width=1\linewidth]{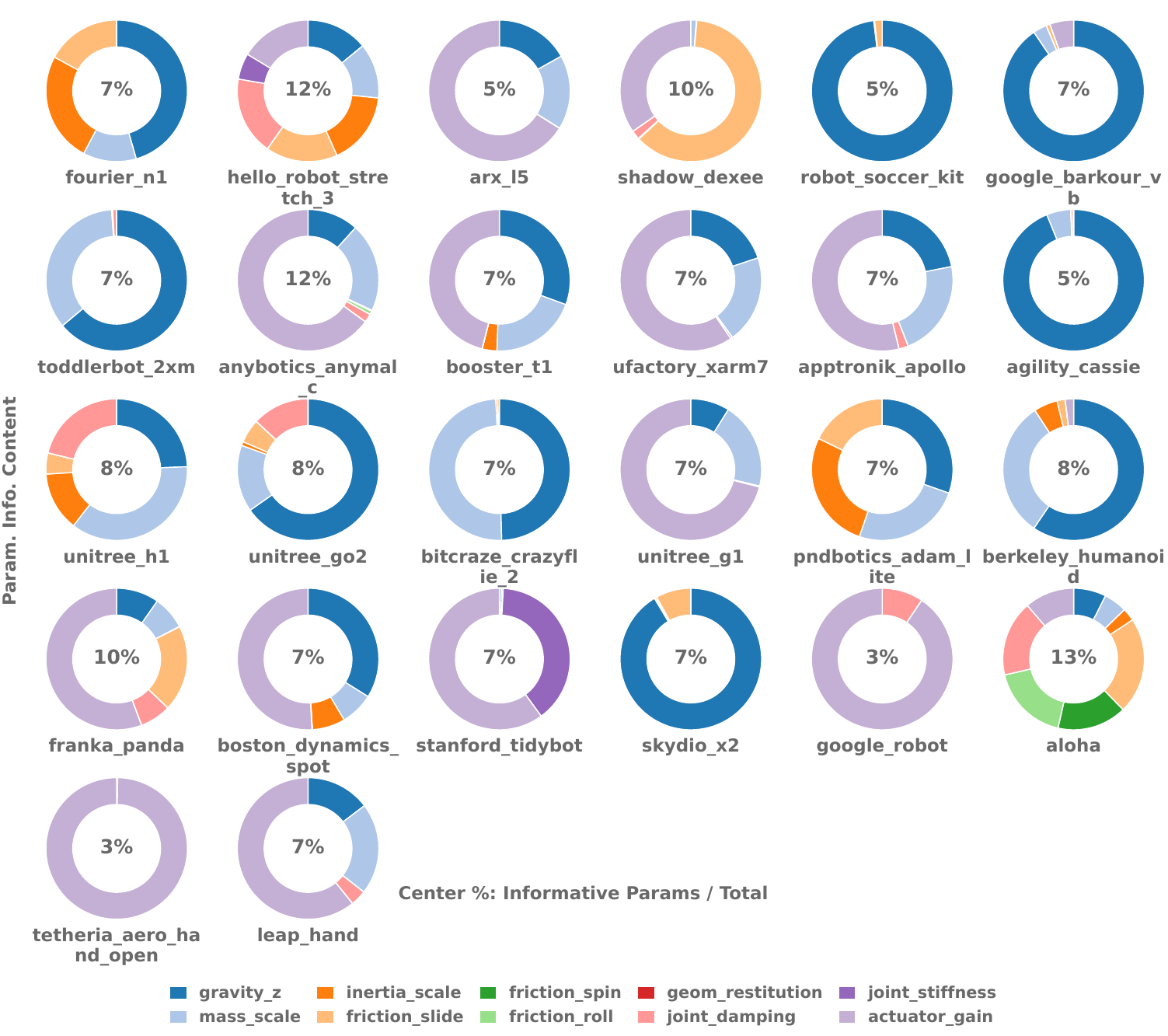}
    \vspace*{-0.2in}
    \caption{\small Fisher information matrix values for physical parameters of \num{26} representative robots. Color denotes parameter type; size indicates information content. The large variation in information distribution indicates that information can be concentrated in only a few critical physical parameters.}
    \label{fig:mujoco-fisher-info}
\end{figure}

\section{Method}
This section details our design of an exploration objective to focus on \emph{critical} parameter directions that are learnable from data.
Sect.~\ref{sec:finding-policy-critical} identifies a compact, non-redundant set of critical parameters, and
Sect.~\ref{sec:qoed} introduces Q{\footnotesize OED}, which prioritizes these parameters while suppressing the influence of the remaining ones.
Sect.~\ref{sec:qoed-mbpo} integrates Q{\footnotesize OED} into model-based policy optimization to improve policy performance.

\subsection{Selecting Critical Parameters}
\label{sec:finding-policy-critical}

To answer \textbf{Q1} presented in Sect.~\ref{subsec:challenges}, we select critical parameters in three steps: (i) estimate the current parameter values, (ii) find the well-observed subspace via eigenspace analysis of the FIM, and (iii) select a compact set of parameter coordinates within the observable subspace that are weakly correlated with each other, improving identifiability.

\textbf{Step \RNum{1} -- Parameter estimation.}
Because the FIM depends on the parameter value, we first estimate $\bm{\phi}$ from data.
Given an observed trajectory prefix $\bm{\tau}_t^{\mathrm{obs}}$ with action sequence
$\mathbf{A}_t = [\bm{a}_0, \dots, \bm{a}_{t-1}]$, we estimate $\bm{\phi}$ by matching rollouts to the observed trajectory:
\begin{equation}
\label{eqn:lse}
    \hat{\bm{\phi}}
    =
    \arg\min_{\bm{\phi}}
    \mathbb{E}_{\bm{\tau}_t \sim p(\bm{\tau}_t \vert \bm{\phi},\mathbf{A}_t)}
    \left[ \| \bm{\tau}_{t}^{\mathrm{obs}} - \bm{\tau}_t \|_2^2 \right].
\end{equation}

Here $p(\bm{\tau}_t \vert \bm{\phi},\mathbf{A}_t)$ is the trajectory distribution with fixed actions $\mathbf{A}_t$, i.e., $ p(\bm{\tau}_t \vert \bm{\phi},\mathbf{A}_t) = p(\mathbf{s}_0)\prod_{k=0}^{t-1} p(\mathbf{s}_{k+1} \vert \mathbf{s}_k, \bm{a}_k, \bm{\phi}) $. See Appendix~\ref{apdx:deri-traj-log-likelihood} for the derivation. Although differentiability would allow gradient-based optimization, we use the cross-entropy method (CEM)~\cite{rubinstein1999cross}, a derivative-free optimizer, which is effective in practice~\cite{memmel2024asid}. We maintain a Gaussian belief over parameters,
$p(\bm{\phi})=\mathcal{N}(\bm{\mu}_{\bm{\phi}},\mathbf{\Sigma}_{\bm{\phi}})$, and update uncertainty approximately as $ \mathbf{\Sigma}_{\bm{\phi}}
\leftarrow
\big(
    \bm{\mathcal{F}}_{\hat{\bm{\phi}}}
    +
    \mathbf{\Sigma}_{\bm{\phi}}^{-1}
\big)^{-1} $, analogous to updates in extended Kalman filtering~\cite{Schmidt-RSS-14}. This Gaussian approximation may overestimate uncertainty, but it empirically facilitates parameter convergence in our setting~\cite{Schmidt-RSS-14, SathyanarayanH-RSS-25}.

\textbf{Step \RNum{2} -- Identify the well-observed subspace.}
We use eigenvalues and eigenvectors of the FIM to identify the critical parts of the parameter space. Eigenvalues indicate how much information the trajectory provides along different directions in parameter space (larger is better). In high-dimensional settings, many FIM eigenvalues are often close to zero~\cite{karakida2019universal}, meaning the data carries little information about many directions (as shown in  Fig.~\ref{fig:mujoco-fisher-info}). Eigenvectors indicate which physical parameters are contributing together in a particular informative direction (in some cases, individual parameters may not be separable: different physical parameters can produce nearly indistinguishable behavior~\cite{WIELAND202160}, which prevents reliable identification of individual parameters). We start with the following assumption for FIM's eigen-decomposition.

\begin{mdframed}[hidealllines=true,backgroundcolor=MaterialBlueGray10!40,innerleftmargin=.2cm,
  innerrightmargin=.2cm]
\begin{assumption}[Regularity]
\label{assump:regularity}
Assume $p(\bm{\tau}_t \vert \bm{\phi},\bm{\pi})$ is continuously differentiable in $\bm{\phi}$ and its score has bounded norm:
\begin{equation*}
    \big\| \nabla_{\bm{\phi}} \log p(\bm{\tau}_t \vert \bm{\phi},\bm{\pi}) \big\|
    \leq \delta_{\mathrm{reg}},
    \qquad \forall \bm{\phi} \in \Phi,
\end{equation*}
for some constant $\delta_{\mathrm{reg}} \ge 0$. This holds under standard smoothness and boundedness conditions; for example, if $\bm{\phi}$ is compact and the score is uniformly bounded.
\end{assumption}
\end{mdframed}

\begin{mdframed}[hidealllines=true,backgroundcolor=MaterialBlue10!40,innerleftmargin=.2cm,
  innerrightmargin=.2cm]
\begin{lemma}
\label{lem:fim-direction}
For any $\bm{\phi} \in \Phi$, the FIM $\bm{\mathcal{F}}_{\bm{\phi}} \in\mathbb{R}^{m\times m}$ is symmetric positive semidefinite and admits an eigen-decomposition
\begin{equation*}
    \bm{\mathcal{F}}_{\bm{\phi}}
    = \mathbf{W} \mathbf{\Lambda} \mathbf{W}^{\top}, 
    \quad
    \mathbf{\Lambda} = \operatorname{diag}(\lambda_1, \dots, \lambda_m),
\end{equation*}
where $ \lambda_1 \ge \dots \ge \lambda_m \ge 0 $ and $\mathbf{W}=[\mathbf{w}_1,\dots,\mathbf{w}_m]\in \mathbb{R}^{m \times m}$ is orthonormal.
The Fisher information along eigen-direction $\mathbf{w}_i$ equals $\lambda_i$:
\begin{equation*}
    \mathbb{E}_{\bm{\tau} \sim p(\bm{\tau} \vert \bm{\phi},\bm{\pi}) }
    \Big[
        \big(
            \nabla_{\bm{\phi}} \log p(\bm{\tau} \vert \bm{\phi},\bm{\pi})^{\top} \mathbf{w}_i
        \big)^2
    \Big]
    = \lambda_i.
\end{equation*}
\end{lemma}

\begin{proof}
Please refer to Appendix~\ref{apdx:proof-fim-direction-lemma}.
\end{proof}
\end{mdframed}

Lemma~\ref{lem:fim-direction} shows that each eigenvalue $\lambda_i$ measures information along eigen-direction $\mathbf{w}_i$.
If $\lambda_i=0$, then the score has zero projection onto $\mathbf{w}_i$, and trajectories provide no local information in that direction.
The Cram\'er--Rao lower bound~\cite{Rao_1947} further implies that small $\lambda_i$ correspond to large lower bounds on estimation error. 

\begin{mdframed}[hidealllines=true,backgroundcolor=MaterialBlueGray10!40,innerleftmargin=.2cm,
  innerrightmargin=.2cm]
\begin{proposition}[Multivariate Cram\'er--Rao lower bound~\cite{Rao_1947}]
\label{thm:cramer-rao}
Let $\hat{\bm{\phi}}$ be any unbiased estimator of the true parameters $\bm{\phi}$.
Under standard regularity conditions,
\begin{equation*}
    \mathrm{Cov}(\hat{\bm{\phi}})
    \succeq
    \bm{\mathcal{F}}_{\bm{\phi}}^{-1},
\end{equation*}
where $\bm{\mathcal{F}}_{\bm{\phi}}$ is the FIM. Consequently,
\begin{equation*}
\label{eq:mse_bound}
    \mathbb{E}  [ \lVert \hat{\bm{\phi}} - \bm{\phi} \rVert^2 ]
    =
    \mathrm{tr}\big( \mathbf{\Sigma}_{\hat{\bm{\phi}}} \big)
    \;\ge\;
    \mathrm{tr}\big( \bm{\mathcal{F}}_{\bm{\phi}}^{-1} \big).
\end{equation*}
\end{proposition}
\end{mdframed}

Let $\delta_{\mathrm{eig}} > 0$ be an eigenvalue threshold and define the index set of well-observed directions
$\mathbf{o} = \{ i \mid \lambda_i \ge \delta_{\mathrm{eig}} \}$, with $ |\mathbf{o}| = n$. We partition the eigen-decomposition as
\begin{equation}
\label{eq:fim-split}
\begin{aligned}
    \mathbf{\Lambda} &=
    \begin{bmatrix}
        \mathbf{\Lambda}_{\mathbf{o}} & \mathbf{0} \\
        \mathbf{0} & \mathbf{\Lambda}_{\overline{\mathbf{o}}}
    \end{bmatrix},
    \
    \mathbf{W} =
    \begin{bmatrix}
        \mathbf{W}_{\mathbf{o}} & \mathbf{W}_{\overline{\mathbf{o}}}
    \end{bmatrix}, \ 
    \mathbf{o} = \{ i \vert \lambda_i \ge \delta_{\mathrm{eig}} \},
\end{aligned}
\end{equation}
where $\mathbf{W}_{\mathbf{o}}=\mathbf{W}[:,\mathbf{o}]\in\mathbb{R}^{m\times n}$ spans the well-observed subspace and
$\mathbf{W}_{\overline{\mathbf{o}}}\in\mathbb{R}^{m\times(m-n)}$ spans the weakly observed (or unobserved) subspace.
Importantly, individual physical parameters may not be observable by themselves; the observable direction is defined as $\mathbf{W}_{\mathbf{o}}^{\top} \bm{\phi}$ in the literature~\cite{ActiveSubspaces}, which represents linear combinations of physical parameters.

\textbf{Step \RNum{3} -- Select identifiable parameter coordinates.}
Even if a direction is well observed, it can correspond to a linear combination of multiple physical parameters. To obtain a compact and interpretable set of coordinates, we select a subset of parameter indices $\mathbf{k}\subseteq \{1,\dots,m\}$ that (i) have strong components in the well-observed subspace and (ii) are not redundant with each other. Let $\mathbf{r}_j^{\top}$ be the $j$-th row of $\mathbf{W}_{\mathbf{o}}$. Intuitively, $\mathbf{r}_j$ describes how parameter $j$ loads onto the well-observed directions. Selecting indices $\mathbf{k}$ gives the row-submatrix $\mathbf{W}_{\mathbf{k}\mathbf{o}} = \mathbf{W}[\mathbf{k},\mathbf{o}]\in\mathbb{R}^{|\mathbf{k}|\times n}$. With a budget $|\mathbf{k}| \le n$, we choose
\begin{equation}
\label{eq:submodular-obj}
\begin{aligned}
    \mathbf{k}
    &=
    \argmax_{\mathbf{k}\subseteq\{1,\dots,m\}:\; |\mathbf{k}| \le n}
       \; \log\det\!\Big(
            \mathbf{W}_{\mathbf{k}\mathbf{o}}\;
            \mathbf{W}_{\mathbf{k}\mathbf{o}}^{\top}
       \Big)
    \\
    &\text{s.t.}\quad
       \big| \cos ( \mathbf{r}_i, \mathbf{r}_j ) \big| \le \delta_{\mathrm{cos}},
       \quad \forall\, i\neq j,\; i,j\in\mathbf{k},
\end{aligned}
\end{equation}
where $\cos(\mathbf{r}_i, \mathbf{r}_j) = \nicefrac{ \mathbf{r}_i^{\top} \mathbf{r}_j }{ \|\mathbf{r}_i\|\|\mathbf{r}_j\| }$.
The log-determinant term favors a diverse set of rows that spans the well-observed subspace, and the cosine constraint prevents selecting highly correlated parameters.
We approximately solve Eq.~\eqref{eq:submodular-obj} with a lazy-greedy procedure~\cite{10.1007/BFb0006528} that adds one index at a time and rejects candidates that violate the cosine constraint.
In the rare case where two parameters induce the same observable direction, this procedure keeps only one representative.

\textbf{Solution to Q1.} Steps \RNum{1}--\RNum{3} output the index set $\mathbf{k}$. We refer to the selected coordinates $\bm{\phi}_{\mathbf{k}} = \bm{\phi}[\mathbf{k}]$ as the \emph{critical (identifiable) parameters}. A straightforward exploration objective is the FIM restricted to these coordinates. Let $\bm{\mathcal{F}}_{\mathbf{k}\mathbf{k}} = \bm{\mathcal{F}}_{\bm{\phi}}[\mathbf{k},\mathbf{k}]$ be the principal submatrix of the FIM. We define
\begin{equation}
\label{eqn:naive-bonus}
    \mathrm{Agnostic\,\,QOED\,\,Objective:} \quad \mathcal{B}_{\text{Agnostic}}(\bm{\phi} \vert \bm{\tau}_t)
    =
    \operatorname{tr}\!\left(\bm{\mathcal{F}}_{\mathbf{k}\mathbf{k}}\right),
\end{equation}
where the term ``Agnostic'' indicates that this objective ignores discarded (i.e., non-critical) parameters.

% \begin{figure}[t!]
%     \centering
%     \includegraphics[width=1\linewidth]{media/demo.pdf}
%     \vspace*{-0.2in}
%     \caption{\small \textbf{Information gain trajectories.} \emph{Top:} Fisher information landscape over mass $\theta_1$ and friction $\theta_2$ with induced trajectories for the box pushing task.
%     \emph{Bottom:} local geometry of the dynamics. Optimizing information in both parameters can drift off the identifiable ridge.
%     Restricting the objective to selected coordinates can stay closer to the ridge but may still excite non-identifiable directions.
%     Our Q{\footnotesize OED} method (Sect.~\ref{sec:qoed}) adds constraints/regularization to suppress these directions and keep trajectories near the ridge.}
%     \label{fig:demo-fisher}
% \end{figure}

\begin{figure}[t!]
    \centering
    \vspace*{0.1in}
    
    {\footnotesize
    \noindent
    \makebox[0.333\linewidth][c]{Vanilla BOED Eq.~\eqref{eq:trace-bonus}}%
    \makebox[0.333\linewidth][c]{Agnostic Q{\footnotesize OED} Eq.~\eqref{eqn:naive-bonus}}%
    \makebox[0.333\linewidth][c]{Q{\footnotesize OED} (ours) Eq.~\eqref{eq:t-optimal-k}}%
    }

    % \vspace{-0.08in}

    \includegraphics[
        width=\linewidth,
        trim=0 0 0 55,
        clip
    ]{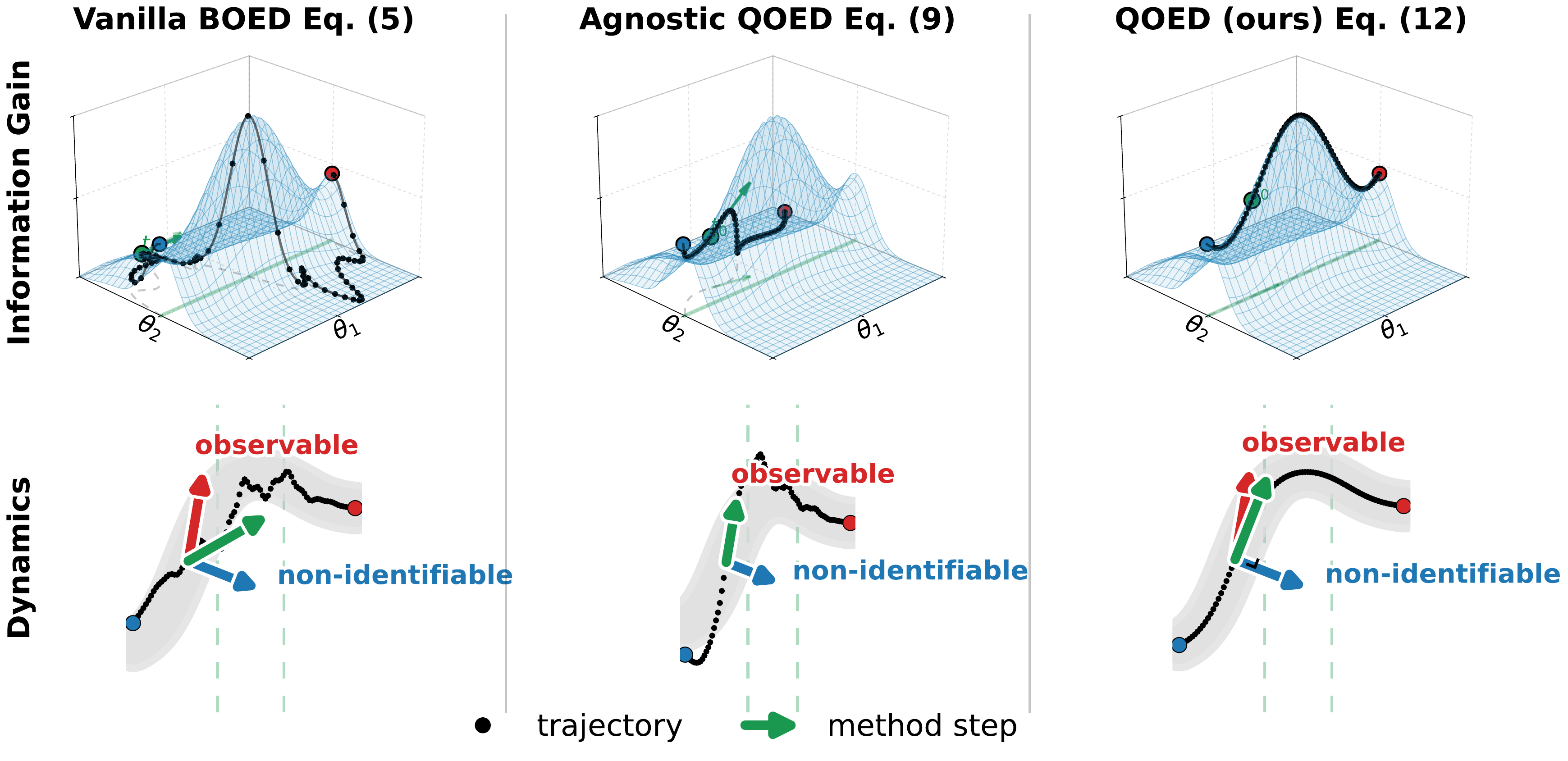}

    % \vspace*{-0.1in}

    \caption{\small
    \textbf{Information gain trajectories.}
    \emph{Top:} Fisher information landscape over mass $\theta_1$ and friction $\theta_2$ with induced trajectories for the box pushing task.
    \emph{Bottom:} local geometry of the dynamics.
    Optimizing information in both parameters can drift off the identifiable ridge.
    Restricting the objective to selected coordinates can stay closer to the ridge but may still excite non-identifiable directions.
    Our Q{\footnotesize OED} method (Sect.~\ref{sec:qoed}) adds constraints/regularization to suppress these directions and keep trajectories near the ridge.
    }

    \label{fig:demo-fisher}
\end{figure}

\subsection{Adaptive Information-Theoretic Objective}
\label{sec:qoed}

We now present our \emph{Quasi-Optimal Experimental Design} (Q{\footnotesize OED}). To answer \textbf{Q2} presented in Sect.~\ref{subsec:challenges}, Q{\footnotesize OED} prioritizes information gain in \emph{critical} parameter coordinates $\bm{\phi}_{\mathbf{k}}$ while \emph{suppressing} the influence of the remaining ones.
This operation is important because {\em non-critical or weakly identifiable parameters can still influence the dynamics and induce non-negligible Fisher information}. In such settings, simply dropping parameters as in Eq.~\eqref{eqn:naive-bonus} can distort the information objective~\cite{ActiveSubspaces}. Formally, denote the score as $\mathbf{g}=[\mathbf{g}_{\mathbf{k}};\mathbf{g}_{\overline{\mathbf{k}}}]$ so that $\bm{\mathcal{F}}_{\bm{\phi}}=\mathbb{E}[\mathbf{g}\mathbf{g}^\top]$ implies the block form in Eq.~\eqref{eq:fim-block-qoed}.
The agnostic objective $\mathcal{B}_{\text{Agnostic}}=\operatorname{tr}(\bm{\mathcal{F}}_{\mathbf{k}\mathbf{k}})$ maximizes the retained
score energy $\mathbb{E}\|\mathbf{g}_{\mathbf{k}}\|_2^2$ while treating $\mathbf{g}_{\overline{\mathbf{k}}}$ as irrelevant.
However, the discarded energy
$\mathbb{E}\|\mathbf{g}_{\overline{\mathbf{k}}}\|_2^2=\operatorname{tr}(\bm{\mathcal{F}}_{\overline{\mathbf{k}}\overline{\mathbf{k}}})$
can be non-negligible, and $\mathbf{g}_{\mathbf{k}}$ can be correlated with $\mathbf{g}_{\overline{\mathbf{k}}}$
(via $\bm{\mathcal{F}}_{\mathbf{k}\overline{\mathbf{k}}}$). As a result, maximizing $\operatorname{tr}(\bm{\mathcal{F}}_{\mathbf{k}\mathbf{k}})$ can overestimate how much the data truly informs
$\bm{\phi}_{\mathbf{k}}$: the apparent information may arise mainly from correlations with nuisance directions
(via $\bm{\mathcal{F}}_{\mathbf{k}\overline{\mathbf{k}}}$).

\textbf{Solution to Q2.}
Let $\mathbf{k}$ denote the \emph{critical parameter indices} returned by Sect.~\ref{sec:finding-policy-critical}, and let $\overline{\mathbf{k}}=\{1,2,\dots,m\}\setminus \mathbf{k}$ be the complementary indices. We block-partition the FIM as
\begin{equation}
\label{eq:fim-block-qoed}
\bm{\mathcal{F}}_{\bm{\phi}}
=
\begin{bmatrix}
\bm{\mathcal{F}}_{\mathbf{k}\mathbf{k}} & \bm{\mathcal{F}}_{\mathbf{k}\overline{\mathbf{k}}} \\
\bm{\mathcal{F}}_{\overline{\mathbf{k}}\mathbf{k}} & \bm{\mathcal{F}}_{\overline{\mathbf{k}}\overline{\mathbf{k}}}
\end{bmatrix}.
\end{equation}

Q{\footnotesize OED} uses the \emph{nuisance-adjusted} Fisher information
for $\bm{\phi}_{\mathbf{k}}$, given by the Schur complement\footnote{We use the standard stabilization
$\bm{\mathcal{F}}_{\overline{\mathbf{k}}\overline{\mathbf{k}}}\leftarrow
\bm{\mathcal{F}}_{\overline{\mathbf{k}}\overline{\mathbf{k}}}+\epsilon \mathbf{I}$ before inversion.}
\begin{equation}
\label{eq:schur-qoed}
\bm{\mathcal{I}}_{\mathbf{k}\vert\overline{\mathbf{k}}}
=
\bm{\mathcal{F}}_{\mathbf{k}\mathbf{k}}
-
\bm{\mathcal{F}}_{\mathbf{k}\overline{\mathbf{k}}}\,
\bm{\mathcal{F}}_{\overline{\mathbf{k}}\overline{\mathbf{k}}}^{-1}\,
\bm{\mathcal{F}}_{\overline{\mathbf{k}}\mathbf{k}}.
\end{equation}

We define the Q{\footnotesize OED} information objective using the trace:
\begin{equation}
\label{eq:t-optimal-k}
\mathrm{QOED\,\,Objective:}\qquad\mathcal{B}_{\text{QOED}}(\bm{\phi}\vert \bm{\tau}_t)
=
\operatorname{tr}\left(
\bm{\mathcal{I}}_{\mathbf{k}\vert\overline{\mathbf{k}}}
\right).
\end{equation}

\begin{algorithm}[h!]
\caption{Q{\footnotesize OED} Information Objective}
\label{alg:explr-bonus}
\SetKwInOut{KwIn}{Input}
\SetKwInOut{KwDefault}{Default}
\SetKwInOut{KwReturn}{Return}

\KwIn{Trajectories $\{ \bm{\tau}_{t,n} \}_{n=1}^{N} $, parameter $ \bm{\phi} $, policy $\bm{\pi}$}
\KwDefault{Thresholds of eigenvalue $\delta_{\mathrm{eig}}=0.1$, $\alpha_{\mathrm{eig}}=0.01$, and similarity $\delta_{\mathrm{cos}}=0.95$}

\tcc{\small Fisher Information Matrix Eq.~\eqref{eqn:fim}}
$\bm{\mathcal{F}}_{\bm{\phi}}(\bm{\pi}) \approx \frac{1}{N}\sum_{n=1}^{N} \mathbf{g}_n \mathbf{g}_n^\top,\quad
\mathbf{g}_n = \nabla_{\bm{\phi}} \log p(\bm{\tau}_{t,n} \vert \bm{\phi},\bm{\pi})$

\tcc{\small eigen-decomposition}
$\mathbf{W}\mathbf{\Lambda}\mathbf{W}^{\top} \gets \bm{\mathcal{F}}_{\bm{\phi}}(\bm{\pi})$

\tcc{\small Observable Subspace Eq.~\eqref{eq:fim-split}}
$\delta_{\mathrm{eig}} \gets \max(\delta_{\mathrm{eig}},\; \alpha_{\mathrm{eig}}\cdot \max \mathrm{diag}(\mathbf{\Lambda}))$

Split $(\mathbf{\Lambda},\mathbf{W})$ into $(\mathbf{\Lambda}_{\mathbf{o}},\mathbf{W}_{\mathbf{o}})$, $(\mathbf{\Lambda}_{\bar{\mathbf{o}}},\mathbf{W}_{\bar{\mathbf{o}}})$ with $\delta_{\mathrm{eig}}$

\tcc{\small Identifiable Parameters Eq.~\eqref{eq:submodular-obj}}
Greedy select $\mathbf{k}$ from $\mathbf{W}_{\mathbf{o}}$ with $ \delta_{\mathrm{cos}} $

$\overline{\mathbf{k}} \gets \{1, 2, \dots, |\bm{\phi}| \} \setminus \mathbf{k}$

\KwReturn{$
\operatorname{tr}\big( \bm{\mathcal{F}}_{\mathbf{k}\mathbf{k}} \big)
- \operatorname{tr} \!\Big(
\bm{\mathcal{F}}_{\mathbf{k}\overline{\mathbf{k}}}
\bm{\mathcal{F}}_{\overline{\mathbf{k}}\overline{\mathbf{k}}}^{-1}
\bm{\mathcal{F}}_{\overline{\mathbf{k}}\mathbf{k}}
\Big)
$}
\end{algorithm}

Alg.~\ref{alg:explr-bonus} summarizes the computation of the Q{\footnotesize OED} information objective. We use a relative threshold $\alpha_{\mathrm{eig}}\cdot \max\mathrm{diag}(\mathbf{\Lambda})$ to make the eigenvalue split robust to global scalings. Here, we also give an interpretation of the Schur complement. Partition the score as $\mathbf{g}=[\mathbf{g}_{\mathbf{k}};\mathbf{g}_{\overline{\mathbf{k}}}]$.
Then the Schur complement $\bm{\mathcal{I}}_{\mathbf{k}\mid\overline{\mathbf{k}}}$ equals the covariance of the residual score after removing
the best linear prediction from $\mathbf{g}_{\overline{\mathbf{k}}}$:
$
\operatorname{tr}(\bm{\mathcal{I}}_{\mathbf{k}\mid\overline{\mathbf{k}}})
=
\min_{\mathbf{A}}\ \mathbb{E}\!\left[\|\mathbf{g}_{\mathbf{k}}-\mathbf{A}\mathbf{g}_{\overline{\mathbf{k}}}\|_2^2\right]
$. Thus Q{\footnotesize OED} rewards information about the critical parameters that cannot be \emph{linearly predicted} from the non-critical ones. As shown in Fig.~\ref{fig:demo-fisher}, Q{\footnotesize OED} stays near the identifiable ridge while optimizing the full trace drifts off the ridge. On the other hand, from an eigen-subspace viewpoint~\cite{ActiveSubspaces}, projecting the score onto a selected FIM eigenspace discards an expected squared score magnitude equal to the sum of the omitted eigenvalues (Appendix~\ref{apdx:proof-score-subspace}). The following theorem shows that optimizing Q{\footnotesize OED} is quasi-optimal, in the sense that it achieves a constant-factor approximation to the full BOED  objective.

\begin{mdframed}[hidealllines=true,backgroundcolor=MaterialBlue10!40,innerleftmargin=.2cm,
  innerrightmargin=.2cm]
\begin{theorem}[Quasi-optimality w.r.t. full BOED]
\label{thm:qoed-quasiopt-trace}
Make the policy dependence explicit by writing
$\bm{\mathcal{F}}^{\bm{\pi}} := \bm{\mathcal{F}}_{\bm{\phi}}$ in Eq.~\eqref{eqn:fim}.
Fix a set of critical indices $\mathbf{k}$ (with complement $\overline{\mathbf{k}}$) and let
$ \bm{\mathcal{I}}^{\bm{\pi}}_{\mathbf{k}\mid\overline{\mathbf{k}}}
:=
\bm{\mathcal{F}}^{\bm{\pi}}_{\mathbf{k}\mathbf{k}}
-
\bm{\mathcal{F}}^{\bm{\pi}}_{\mathbf{k}\overline{\mathbf{k}}}\,
 \left( \bm{\mathcal{F}}^{\bm{\pi}}_{\overline{\mathbf{k}}\overline{\mathbf{k}}} \right)^{-1}\,
\bm{\mathcal{F}}^{\bm{\pi}}_{\overline{\mathbf{k}}\mathbf{k}} $. Define the full BOED and Q{\footnotesize OED} objectives
\begin{equation}
\label{eq:objectives-trace-main}
\mathcal{B}_{\text{BOED}}(\bm{\pi}) := \tr\!\big(\bm{\mathcal{F}}^{\bm{\pi}}\big),
\qquad
\mathcal{B}_{\text{QOED}}(\bm{\pi}) := \tr\!\big(\bm{\mathcal{I}}^{\bm{\pi}}_{\mathbf{k}\mid\overline{\mathbf{k}}}\big).
\end{equation}

Let $\bm{\pi}^\star \in \argmax_{\bm{\pi}\in\Pi}\mathcal{B}_{\text{BOED}}(\bm{\pi})$
and $\widehat{\bm{\pi}} \in \argmax_{\bm{\pi}\in\Pi}\mathcal{B}_{\text{QOED}}(\bm{\pi})$.
Assume $\bm{\mathcal{F}}^{\bm{\pi}}_{\mathbf{k}\mathbf{k}}\succ \mathbf{0}$ and
$\bm{\mathcal{F}}^{\bm{\pi}}_{\overline{\mathbf{k}}\overline{\mathbf{k}}}\succ \mathbf{0}$ for all $\bm{\pi}\in\Pi$ and define
\begin{equation}
\label{eq:eta-beta-main}
\begin{aligned}
\eta
&=
\sup_{\bm{\pi}\in\Pi}
\frac{\tr(\bm{\mathcal{F}}^{\bm{\pi}}_{\overline{\mathbf{k}}\overline{\mathbf{k}}})}{\tr(\bm{\mathcal{F}}^{\bm{\pi}}_{\mathbf{k}\mathbf{k}})},
\\
\beta
&=
\sup_{\bm{\pi}\in\Pi}
\left\|
\big(\bm{\mathcal{F}}^{\bm{\pi}}_{\mathbf{k}\mathbf{k}}\big)^{-1/2}
\bm{\mathcal{F}}^{\bm{\pi}}_{\mathbf{k}\overline{\mathbf{k}}}
\big(\bm{\mathcal{F}}^{\bm{\pi}}_{\overline{\mathbf{k}}\overline{\mathbf{k}}}\big)^{-1/2}
\right\|_2^2.
\end{aligned}
\end{equation}

If $\eta<\infty$ and $\beta<1$, then the Q{\footnotesize OED}-optimal policy is a constant-factor approximation to the full BOED-optimal policy:
\begin{equation}
\label{eq:quasiopt-main}
\tr\!\big(\bm{\mathcal{F}}^{\widehat{\bm{\pi}}}\big)
\;\ge\;
\frac{1-\beta}{1+\eta}\;
\tr\!\big(\bm{\mathcal{F}}^{\bm{\pi}^\star}\big).
\end{equation}
\end{theorem}
\begin{proof}
Please refer to Appendix~\ref{apdx:proof-quasi-optimal} for a justification of $\eta$ and $\beta$, an extension to our time-varying $\mathbf{k}$, and a proof that the agnostic objective is not quasi-optimal in general.
\end{proof}
\end{mdframed}

\subsection{Policy Learning with Q{\footnotesize OED}}
\label{sec:qoed-mbpo}
We consider Problem~\ref{problem:mpc} in a model-based policy optimization (MBPO) setting~\cite{janner2019mbpo}.
Because analytical physics models can be mismatched to real dynamics, we learn a differentiable dynamics model
\begin{equation}
\label{eq:demo-transition}
q_{\bm{\theta}}\!\left(\mathbf{s}_{t+1}\vert \mathbf{s}_t,\bm{a}_t,\bm{\phi}\right),
\end{equation}
parameterized by $\bm{\theta}$. Together with a policy $\bm{\pi}(\bm{a}\vert \mathbf{s})$, this induces a trajectory likelihood $q_{\bm{\theta}}(\bm{\tau}_t\vert \bm{\phi},\bm{\pi})$ that serves as a surrogate for the real likelihood $p(\bm{\tau}_t\vert \bm{\phi},\bm{\pi})$.
We instantiate $q_{\bm{\theta}}$ with a shortcut model~\cite{frans2025one}: we sample latent noise $\bm{\delta}\sim\mathcal{N}(\bm{0},\bm{I})$ and generate a one-step state increment via a transport map $T_{\bm{\theta}}$,
\begin{equation}
\label{eq:demo-sm-gen}
\bm{\delta}\sim\mathcal{N}(\bm{0},\bm{I}),
\qquad
\mathbf{s}_{t+1}
=
\mathbf{s}_t
+
T_{\bm{\theta}}\!\left(\bm{\delta},\mathbf{s}_t,\bm{a}_t,\bm{\phi}\right).
\end{equation}

We train $T_{\bm{\theta}}$ using the shortcut-model objective~\cite{frans2025one} (Appendix~\ref{apdx:fmpe}). If $T_{\bm{\theta}}$ is invertible in
$\bm{\delta}$, then the conditional log-density follows by change of variables~\cite{chen2018neural}:
\begin{equation}
\label{eq:demo-sm-loglik}
\log q_{\bm{\theta}}(\mathbf{s}_{t+1}\vert \mathbf{s}_t,\bm{a}_t,\bm{\phi})
=
\log p_0(\bm{\delta})
-
\log\left|
\det \nabla_{\bm{\delta}}
T_{\bm{\theta}}(\bm{\delta},\mathbf{s}_t,\bm{a}_t,\bm{\phi})
\right|,
\end{equation}
where $p_0$ is the standard normal density and $\bm{\delta}$ is the latent satisfying
$\mathbf{s}_{t+1}-\mathbf{s}_t=T_{\bm{\theta}}(\bm{\delta},\mathbf{s}_t,\bm{a}_t,\bm{\phi})$.
This yields a tractable estimate of the FIM; the trajectory-level derivation is in Appendix~\ref{apdx:deri-traj-log-likelihood}, following Eq.~\eqref{eq:qom-fim}. We parameterize $T_{\bm{\theta}}$ with a Transformer~\cite{NIPS2017_3f5ee243}: state, action, and parameter inputs are embedded by separate MLPs, processed by a six-block Transformer with Adaptive Layer Normalization (AdaLN), and mapped to the output by a final AdaLN modulation layer. Architecture details are in Appendix~\ref{apdx:model-struc}.

\begin{table*}[b!]
\centering
\footnotesize
\setlength{\tabcolsep}{9pt}
\renewcommand{\arraystretch}{1.0}

\begin{tabular}{@{}llcccccc@{}}
\toprule
\multirow{2}{*}{Env.} & \multirow{2}{*}{Noise} &
\multicolumn{2}{c}{\textbf{Q{\footnotesize OED} (Ours)}} &
\multicolumn{2}{c}{\textbf{Q{\footnotesize OED}-Agnostic}} &
\multicolumn{2}{c}{\textbf{BOED}} \\
\cmidrule(lr){3-4}\cmidrule(lr){5-6}\cmidrule(lr){7-8}
& &
\textbf{Param.\ Est.\,$\downarrow$} & \textbf{Dyn.\ Pred.\,$\downarrow$} &
\textbf{Param.\ Est.\,$\downarrow$} & \textbf{Dyn.\ Pred.\,$\downarrow$} &
\textbf{Param.\ Est.\,$\downarrow$} & \textbf{Dyn.\ Pred.\,$\downarrow$} \\
\midrule

\multirow{3}{*}{\texttt{Go1}}
& $1\sigma$ & $\textcolor{gray}{848.29\pm26.95}$ & $\bm{28.57\pm4.56}$
& $\bm{847.95\pm27.0}$ & $59.64\pm18.08$
& {$848.15\pm26.92$} & \textcolor{gray}{$119.04\pm94.74$} \\
& $2\sigma$ & $\bm{847.89\pm27.0}$ & $\bm{32.22\pm4.97}$
& $\textcolor{gray}{848.08\pm26.99}$ & \textcolor{gray}{$73.16\pm35.84$}
& {$847.91\pm27.0$} & {$59.36\pm22.39$} \\
& $3\sigma$ & \textcolor{gray}{$847.98\pm26.99$} & $\bm{37.91\pm5.17}$
& ${833.94\pm25.42}$ & $47.19\pm10.54$
& \bm{$825.69\pm27.69$} & \textcolor{gray}{$101.88\pm9.51$} \\
\addlinespace[1ex]

\multirow{3}{*}{\texttt{G1}}
& $1\sigma$ & $\bm{963.02\pm33.73}$ & $\bm{37.65\pm1.24}$
& \textcolor{gray}{$1059.73\pm30.58$} & $39.52\pm1.75$
& $1014.05\pm56.02$ & \textcolor{gray}{$44.93\pm2.43$} \\
& $2\sigma$ & $\bm{1051.11\pm35.28}$ & $\bm{35.61\pm0.88}$
& \textcolor{gray}{$1063.02\pm33.95$} & $37.18\pm1.02$
& $1062.99\pm33.74$ & \textcolor{gray}{$40.73\pm1.23$} \\
& $3\sigma$ & $\textcolor{gray}{1063.17\pm33.72}$ & $\bm{35.49\pm0.70}$
& {$1063.11\pm33.75$} & $43.46\pm1.27$
& \bm{$1063.05\pm33.73$} & \textcolor{gray}{$44.52\pm1.45$} \\
\addlinespace[1ex]

\multirow{3}{*}{\texttt{Jackal}}
& $1\sigma$ & $\bm{291.81\pm75.29}$ & $\bm{1.00\pm0.13}$
& $369.74\pm102.23$ & \textcolor{gray}{$1.02\pm0.11$}
& \textcolor{gray}{$370.74\pm101.87$} & $1.01\pm0.08$ \\
& $2\sigma$ & $\bm{296.90\pm66.54}$ & $\bm{1.64\pm0.11}$
& $348.36\pm91.93$ & $1.73\pm0.16$
& \textcolor{gray}{$363.12\pm99.15$} & \textcolor{gray}{$6.57\pm3.34$} \\
& $3\sigma$ & $\bm{293.43\pm63.38}$ & $\bm{2.64\pm0.25}$
& $360.77\pm87.60$ & $4.92\pm2.40$
& \textcolor{gray}{$375.36\pm95.16$} & \textcolor{gray}{$5.29\pm2.38$} \\
\addlinespace[1ex]

\multirow{3}{*}{\texttt{Hand}}
& $1\sigma$ & $\bm{50.98\pm3.89}$ & $\bm{3.55\pm0.57}$
& {$52.68\pm2.72$} & $4.66\pm0.68$
& \textcolor{gray}{$52.68\pm2.84$} & \textcolor{gray}{$4.75\pm0.69$} \\
& $2\sigma$ & $\bm{51.74\pm3.7}$ & $\bm{3.86\pm0.55}$
& $52.59\pm2.99$ & $4.57\pm0.63$
& \textcolor{gray}{$52.67\pm2.72$} & \textcolor{gray}{$4.72\pm0.67$} \\
& $3\sigma$ & $\bm{52.59\pm2.99}$ & $\bm{4.8\pm0.61}$
& \textcolor{gray}{$53.18\pm2.91$} & \textcolor{gray}{$5.16\pm0.69$}
& ${52.68\pm2.72}$ & ${5.04\pm0.7}$ \\
\bottomrule
\end{tabular}
\vspace*{-0.05in}
\caption{\small Parameter-estimation and dynamics-prediction RMSE ($\times 100$), averaged over \num{25} seeds and (when available) over \texttt{Flat} and \texttt{Rough} environments.
Lower is better; best results are in \textbf{bold}, and least favorable results are in \textcolor{gray}{gray}.}
\label{tab:sim-physic-params}
\end{table*}

To learn the policy, we follow the MBPO paradigm~\cite{li2025robotic} using a learned dynamics model $q_{\bm{\theta}}$ to train a PPO~\cite{schulman2017proximal} policy $\bm{\pi}$. Specifically, we follow the Robotic World Model~\cite{li2025robotic} pipeline. To prevent catastrophic failure and initialize the dynamics model with physics knowledge, we pretrain both the policy $\bm{\pi}$ and the dynamics model $q_{\bm{\theta}}$ in simulation with domain randomization of physics parameters~\cite{chen2022understanding}. Upon deployment, we transition to online learning to bridge the reality gap. At each iteration:
\begin{enumerate}[leftmargin=*]
    \item Collect a transition $(\mathbf{s}_t,\bm{a}_t,\hat{\bm{\phi}},\mathbf{s}_{t+1})$ with the current policy $\bm{\pi}$, where $\hat{\bm{\phi}}$ is estimated using Eq.~\eqref{eqn:lse}, and update
$\mathcal{D}\leftarrow \mathcal{D}\cup\{(\mathbf{s}_t,\bm{a}_t,\hat{\bm{\phi}},\mathbf{s}_{t+1})\}$.
    \item Update dynamics $ q_{\bm{\theta}} $ with the shortcut model objective~\cite{frans2025one} using data sampled from dataset $\mathcal{D}$. 
    \item Roll out imagined trajectories with domain randomization branched from states in $\mathcal{D}$ under $q_{\bm{\theta}}$, $\bm{\pi}$, and $p(\bm{\phi})$. Update $\bm{\pi}$ with rewards augmented by Alg.~\ref{alg:explr-bonus}.
\end{enumerate}

Exploration terminates when the trace of the parameter posterior covariance falls below $\delta_{\mathrm{var}}$ and the dynamics prediction error falls below $\delta_{\mathrm{dyn}}$. The prediction error is measured as the horizon-$H$ $\ell_2$ deviation between real states and model rollouts.

\section{Simulation Experiment}
\label{sec:sim-exp}

We evaluate the two questions raised in the preliminary section. \textbf{(Q1)} Does the Q{\footnotesize OED} strategy collect information that improves identification of unknown parameters? \textbf{(Q2)} Does Q{\footnotesize OED} improve policy learning compared with state-of-the-art methods? What is the contribution of each component? Finally, we compare the learned dynamics model against the purely analytical physics model.

\textbf{Environments.}
We evaluate on seven environments in MuJoCo~\cite{todorov2012mujoco}, covering locomotion and manipulation platforms: Unitree Quadruped \texttt{Go1-Flat} and \texttt{Go1-Rough} (\num{54}), Humanoid \texttt{G1-Flat} and \texttt{G1-Rough} (\num{119}), Clearpath Vehicle \texttt{Jackal-Flat} and \texttt{Jackal-Rough} (\num{31}), and the dexterous Inspire-FTP \texttt{Hand-Rotate} (\num{104}). The numbers in parentheses denote the dimensionality of the physical parameters. \texttt{Flat} and \texttt{Rough} indicate flat ground and uneven terrain, respectively. All policies are pretrained in mjlab~\cite{mjlab} with \num{100} episodes, using the mjlab default configuration, and then run in MuJoCo with \num{20} episodes, on a single NVIDIA 4090 GPU. In MuJoCo, we use a single robot as a surrogate for real-world challenges, and we randomize physics coefficients by sampling from mjlab's default domain randomization distribution. We additionally test multiple noise levels, with the first level set to $\sigma = 0.025$. Full settings are provided in Appendix~\ref{apdx:model-struc}.

\subsection{Does Q{\footnotesize OED} yield informative data?}
In the MuJoCo evaluation, we analyze \textbf{Q1} using the RMSE of (i) parameter estimates and (ii) dynamics predictions. We compare Q{\footnotesize OED} against two ablations. \textbf{(1) Q{\footnotesize OED}-A{\footnotesize GNOSTIC}} only considers identifiable parameters~\cite{memmel2024asid,sobanbabu2025sampling}, defined in Eq.~\eqref{eqn:naive-bonus} as $ \mathcal{B}_{\text{Agnostic}}(\bm{\phi} \vert \bm{\tau}_t) = \tr\left(\bm{\mathcal{F}}_{\mathbf{k}\mathbf{k}}\right)$, which is agnostic to the discarded parameters. \textbf{(2) B{\footnotesize OED}} is the canonical and ideal formulation that uses all parameters, defined in Eq.~\eqref{eq:trace-bonus} as $\mathcal{B}_{\text{BOED}}(\bm{\phi} \vert \bm{\tau}_t) = \tr\left(\bm{\mathcal{F}}_{\bm{\phi}}\right)$. We do not include an ``observable-subspace'' BOED variant because the observable directions $\mathbf{W}_{\mathbf{o}}^{\top}\bm{\phi}$ need not map uniquely to physical parameters. All methods use the same CEM optimizer and learned dynamics model, running \num{5} optimization steps with \num{2048} samples per step. Following~\cite{MurilloGonzalezA-RSS-25}, we use $\alpha=1$, $H=\SI{2}{\second}$, $\delta_{\mathrm{var}}=0.05^2$, and $\delta_{\mathrm{dyn}}=1$ as defaults.

\begin{figure*}[t!]
  \centering
  
  % --- Top Row: Simulation Snapshots ---
  \begin{subfigure}[b]{0.135\textwidth}
    \centering
    \mycroppedimg{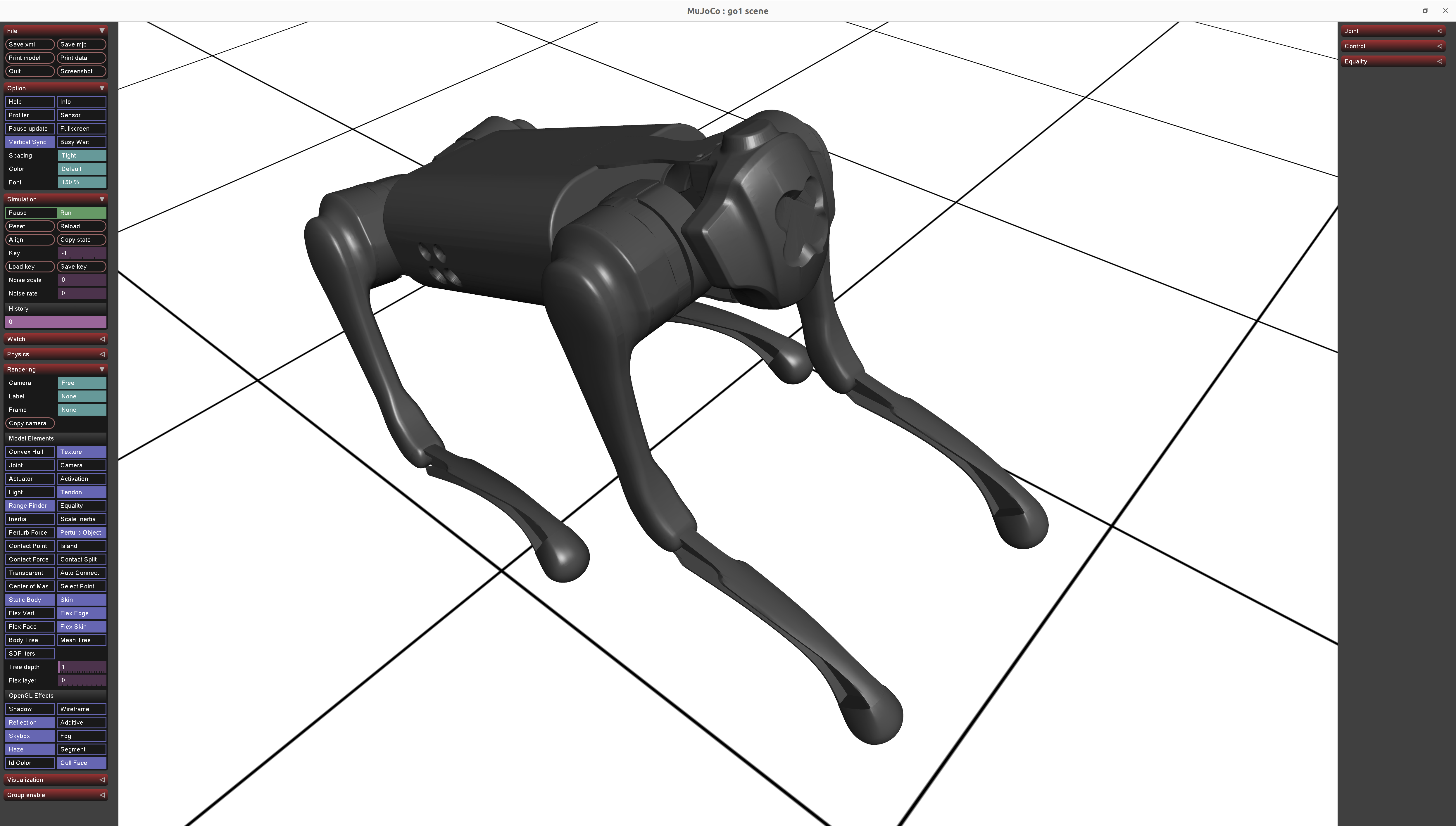}
    \caption*{Go1 Flat}
    \label{fig:ex1}
  \end{subfigure}
  \hfill
  \begin{subfigure}[b]{0.135\textwidth}
    \centering
    \mycroppedimg{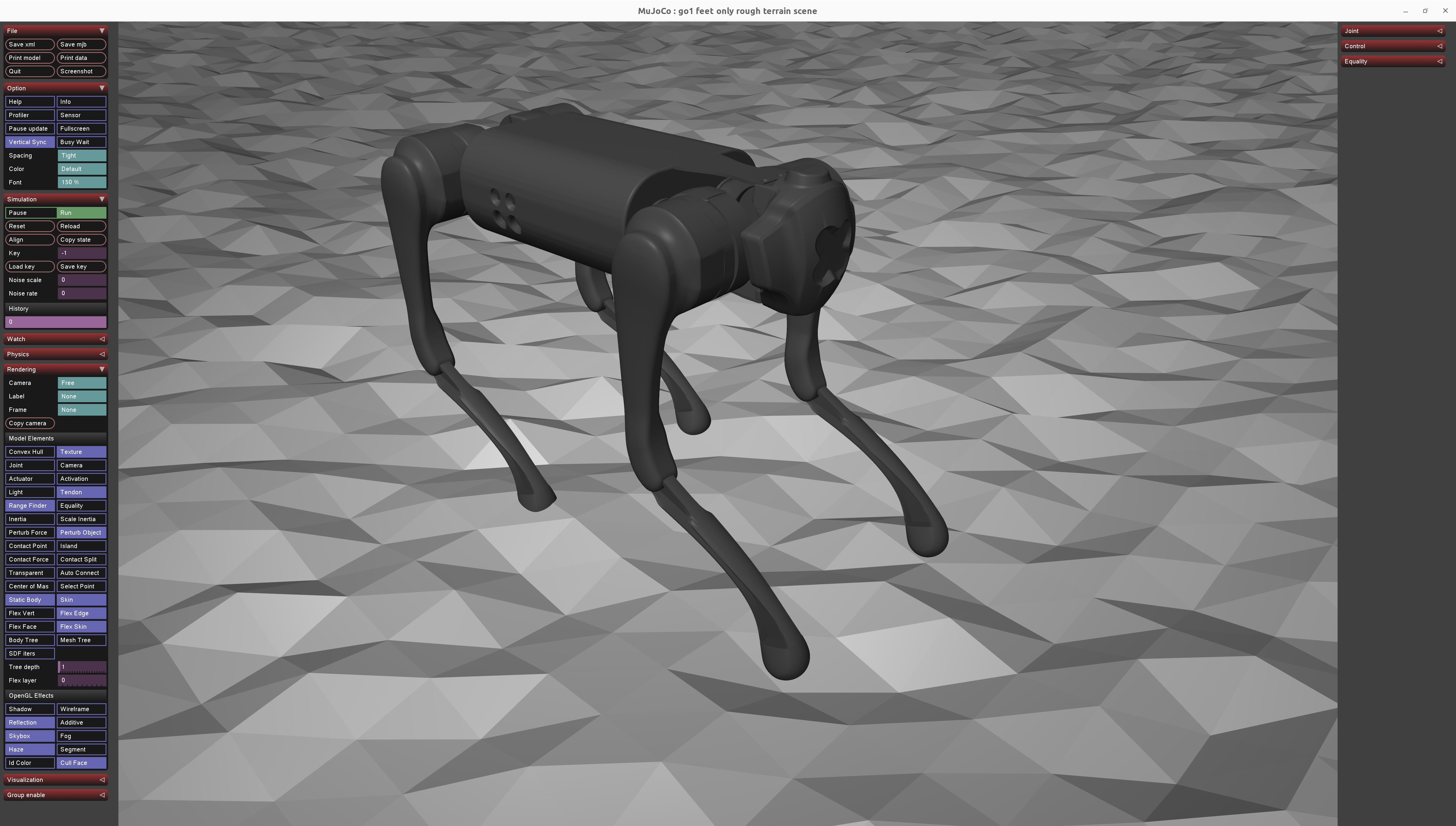}
    \caption*{Go1 Rough}
    \label{fig:ex2}
  \end{subfigure}
  \hfill
  \begin{subfigure}[b]{0.135\textwidth}
    \centering
    \mycroppedimg{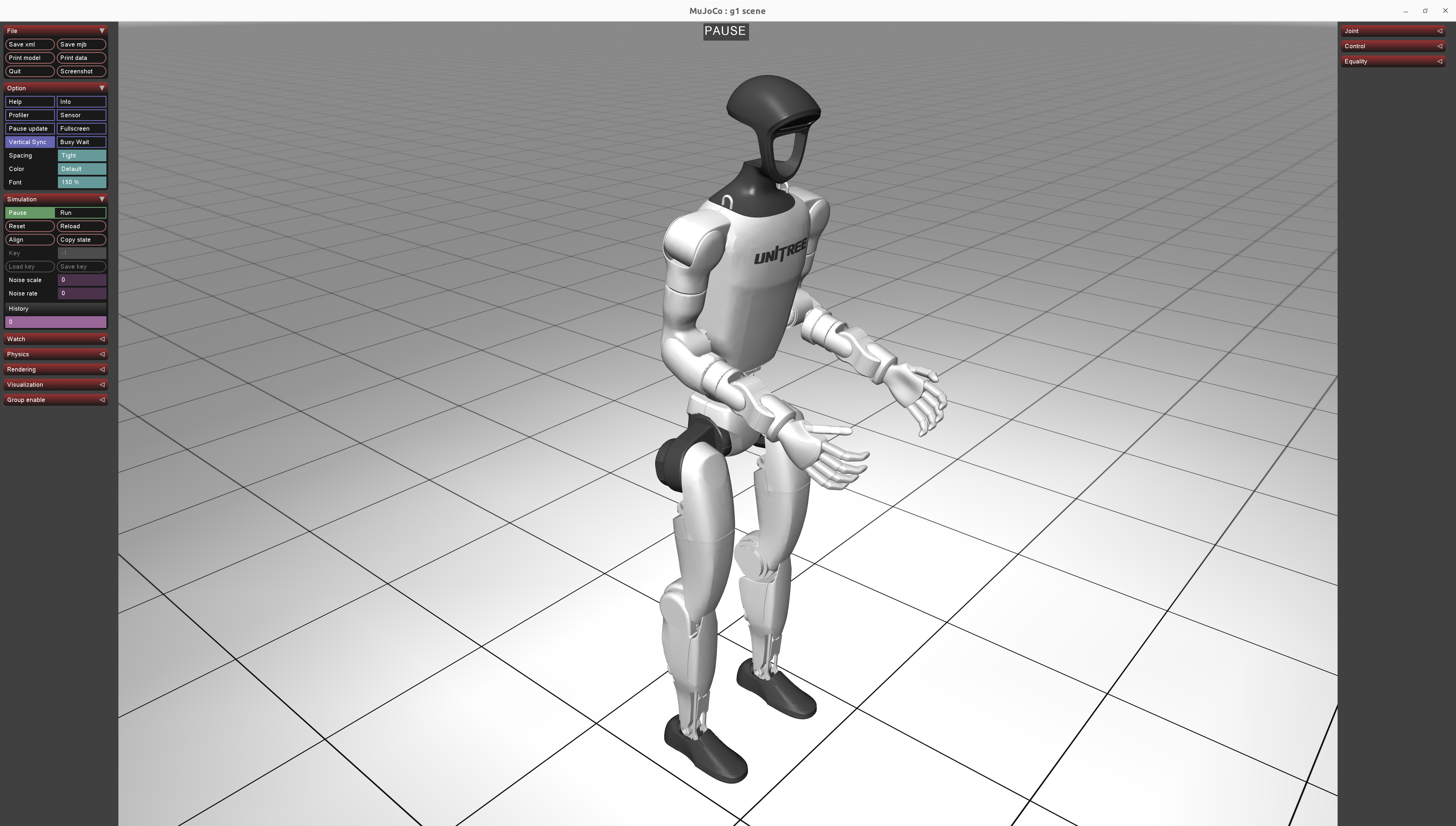}
    \caption*{G1 Flat}
    \label{fig:ex3}
  \end{subfigure}
  \hfill
  \begin{subfigure}[b]{0.135\textwidth}
    \centering
    \mycroppedimg{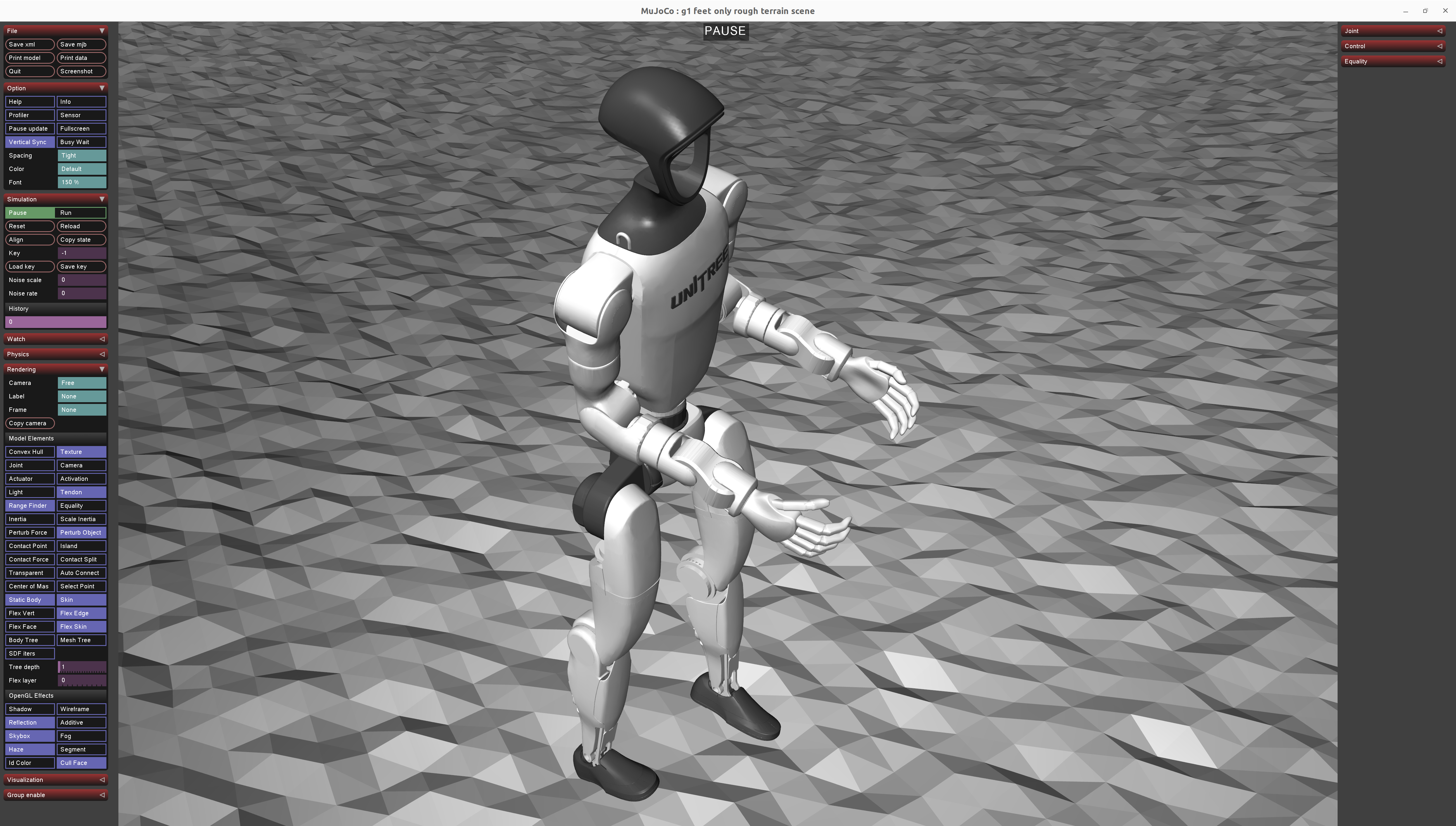}
    \caption*{G1 Rough}
    \label{fig:ex4}
  \end{subfigure}
  \hfill
  \begin{subfigure}[b]{0.135\textwidth}
    \centering
    \mycroppedimg{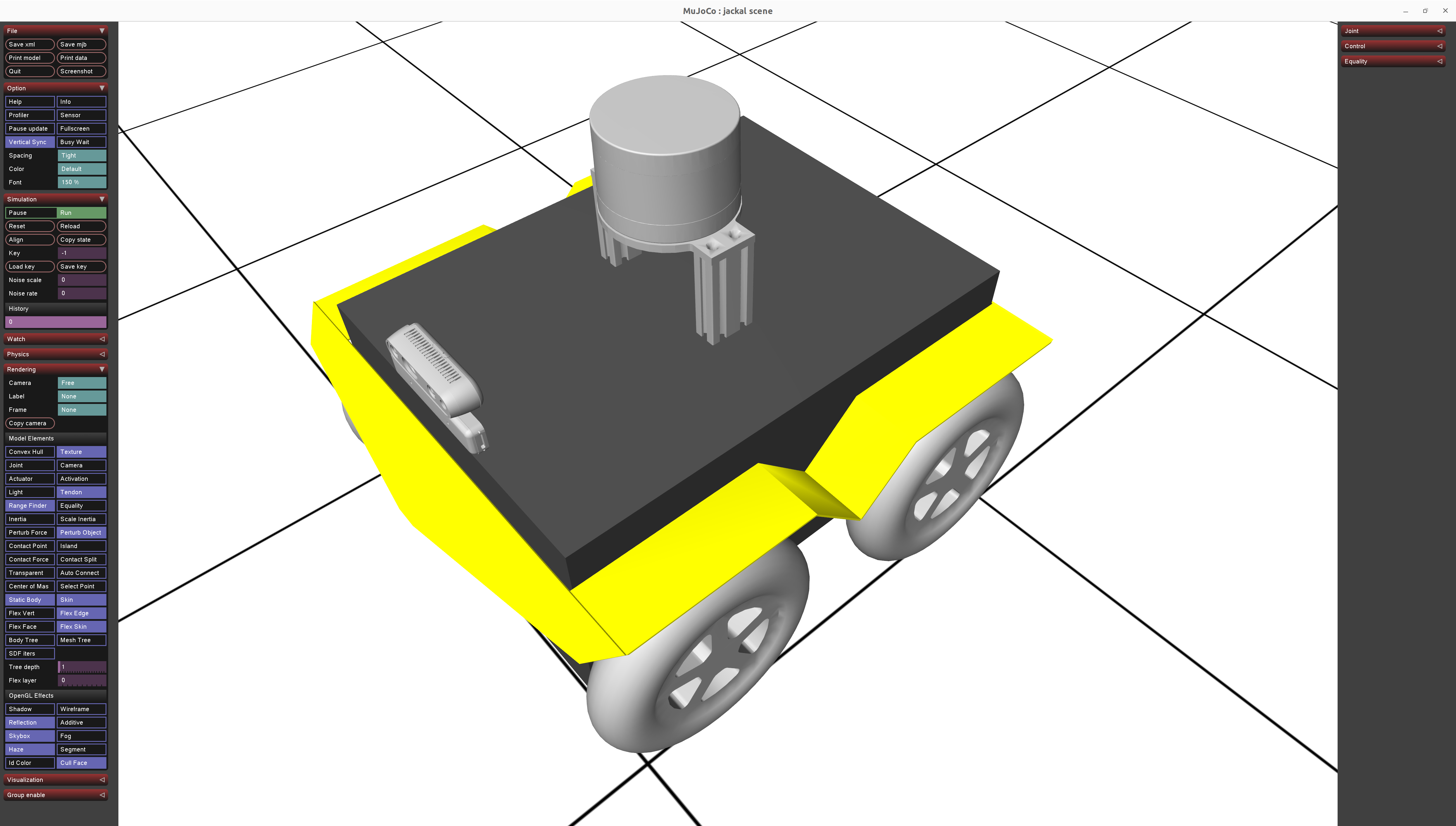}
    \caption*{Jackal Flat}
    \label{fig:ex5}
  \end{subfigure}
  \hfill
  \begin{subfigure}[b]{0.135\textwidth}
    \centering
    \mycroppedimg{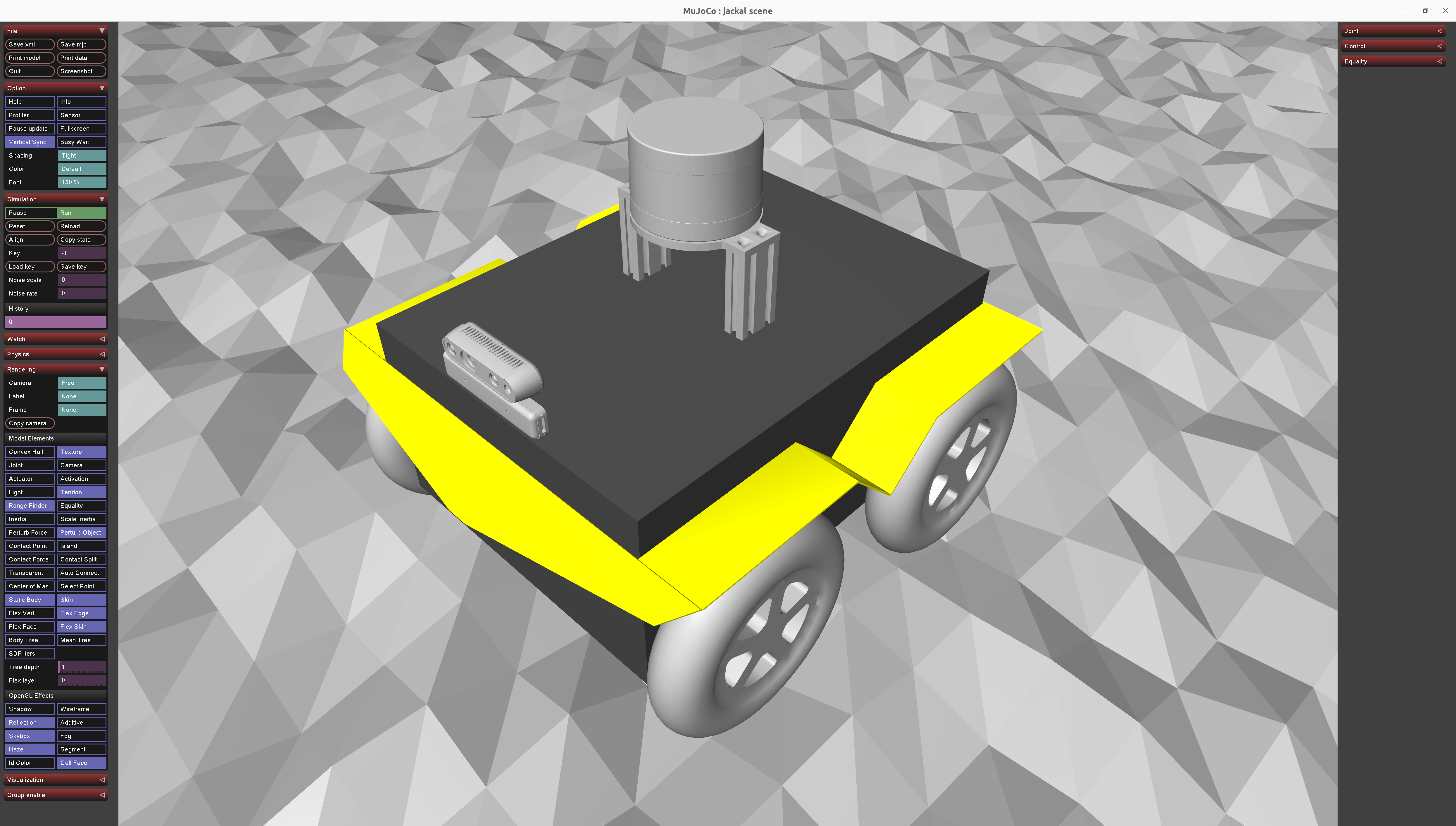}
    \caption*{Jackal Rough}
    \label{fig:ex6}
  \end{subfigure}
  \hfill
  \begin{subfigure}[b]{0.135\textwidth}
    \centering
    \mycroppedimg{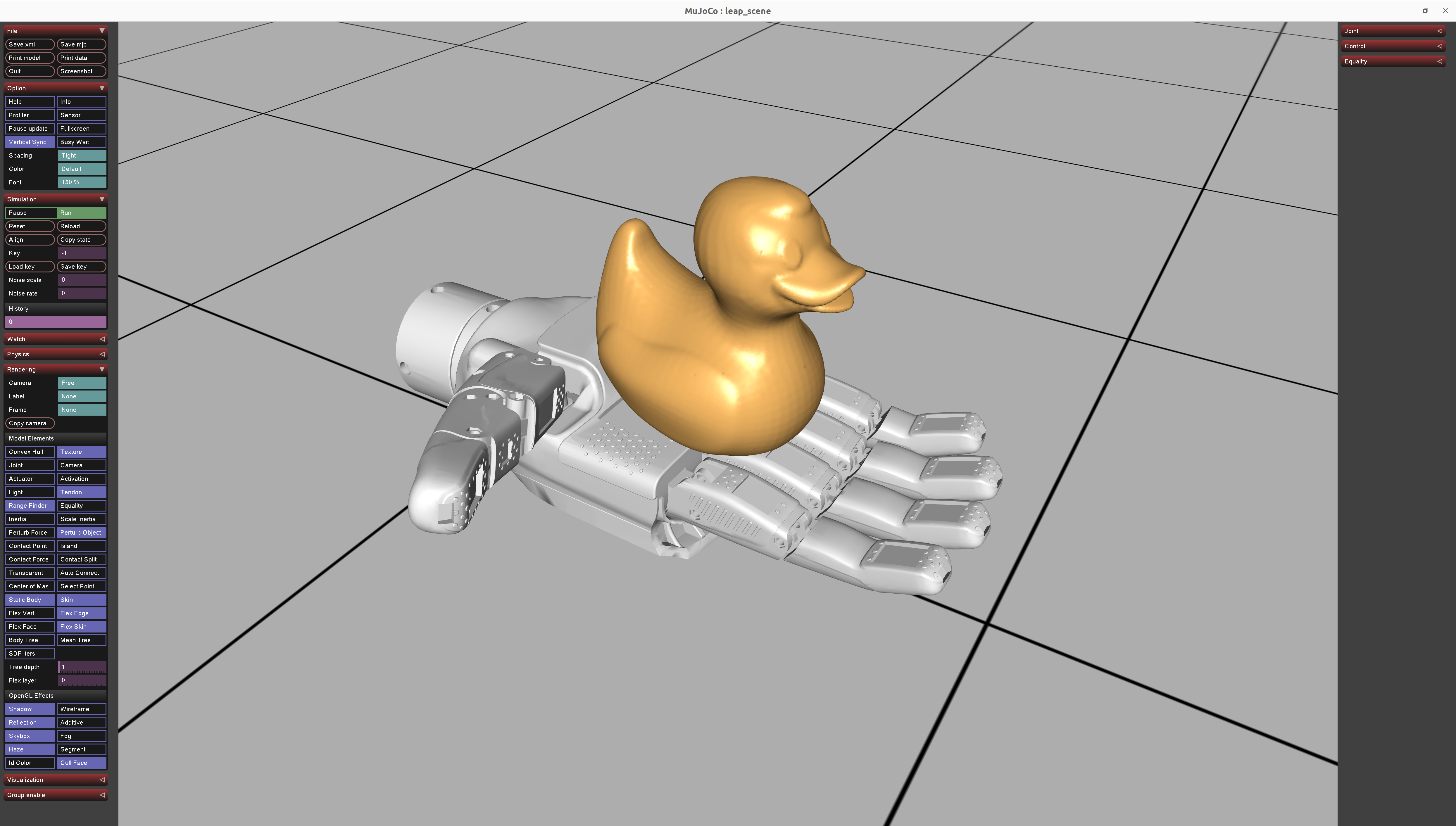}
    \caption*{Hand Rotate}
    \label{fig:ex7}
  \end{subfigure}

  % --- Vertical Spacing ---
  \vspace*{0.05in} 

  % --- Bottom Row: Policy Plot ---
  \includegraphics[width=0.985\textwidth]{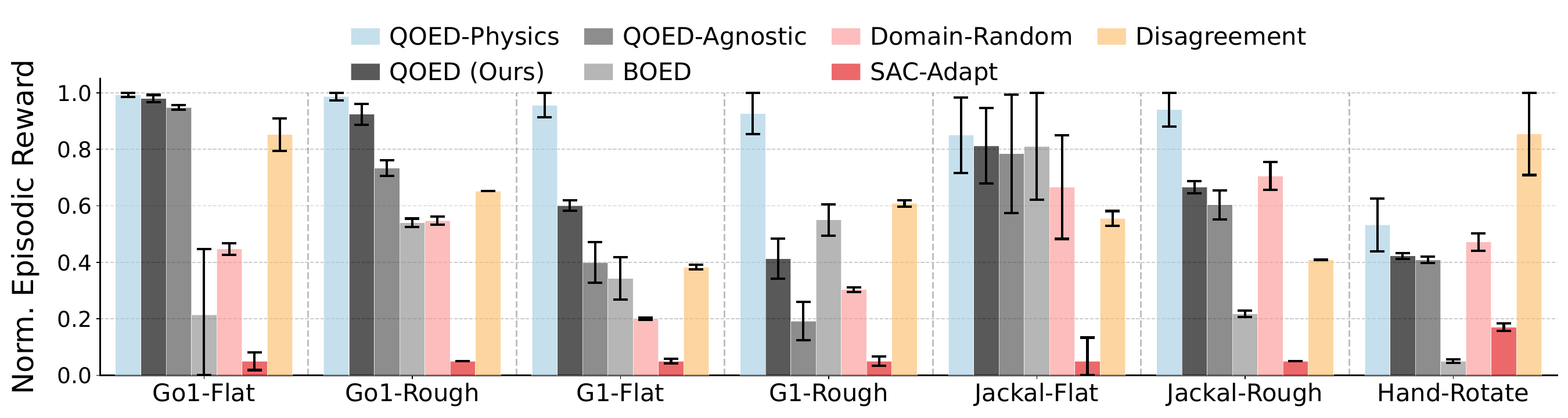}
  
  \vspace*{-0.1in}
  \caption{\small Policy performance across diverse robot environments. Q{\footnotesize OED}-P{\footnotesize HYSICS} with ground-truth physics consistently outperforms baselines, validating our adaptive information objective. Q{\footnotesize OED} with learned dynamics also performs well, highlighting the promise of learned models for exploration.}
  \label{fig:sim-policy}
\end{figure*}

\textbf{Results.} 
Table~\ref{tab:sim-physic-params} reports (i) the RMSE of parameter estimates and (ii) the RMSE of dynamics prediction when using the within-episode parameter estimate, averaged over \num{25} random seeds per scenario. Across all scenarios and noise levels, Q{\footnotesize OED} achieves the strongest dynamics-prediction RMSE and competitive parameter-estimation RMSE. Q{\footnotesize OED} improves dynamics prediction by $\bm{21.98\%}$ over Q{\footnotesize OED}-A{\footnotesize GNOSTIC} and by $\bm{35.23\%}$ over B{\footnotesize OED}. B{\footnotesize OED} often produces poor estimates because it does not explicitly account for observability and can allocate exploration effort to directions that are hard to learn. Q{\footnotesize OED}-A{\footnotesize GNOSTIC} restricts exploration to the identifiable subspace, but it
still underperforms Q{\footnotesize OED} because it treats the discarded parameters as irrelevant; in practice, these dropped directions can remain coupled to the critical ones and degrade estimation. Although the parameter estimation RMSE is similar across methods, the prediction RMSE differs because not all parameters contribute equally to the dynamics. Q{\footnotesize OED} gathers data that is informative about the critical parameters while reducing the influence of the dropped ones, enabling CEM to recover the parameters that matter most for accurate prediction. Our CEM runtime averages \SI{6.3}{\milli\second} and FIM estimation runtime averages \SI{28.3}{\milli\second}. Overall, the results highlight the importance of identifying critical parameters and suppressing the influence of discarded directions.

\textbf{Why dynamics prediction improves despite similar parameter errors?}
To explain why dynamics-prediction RMSE can improve even when aggregate parameter RMSE changes less, we perform a post-hoc attribution analysis. Specifically, by fixing all non-critical parameters to their ground-truth values, we measure how much of the dynamics-prediction error is explained by the identified critical parameter set. The critical parameters account for $65.9\%$ (Go1), $53.1\%$ (G1), $51.9\%$ (Jackal), and $47.2\%$ (Hand) of the total dynamics-prediction error. For G1, for example, Q{\footnotesize OED} identifies only \num{16} parameters ($13.4\%$ of the full parameter space), indicating that a compact identifiable subset can explain a disproportionate share of dynamics prediction.

\textbf{Robustness to Hyperparameters.}
We assess Q{\footnotesize OED} robustness to $\delta_{\mathrm{eig}}$, $\alpha_{\mathrm{eig}}$, and $\delta_{\mathrm{cos}}$ using a grid search. We sweep $\delta_{\mathrm{eig}} \in [0.05,0.5]$ (step $0.05$), $\alpha_{\mathrm{eig}} \in [0.005,0.05]$ (step $0.005$), and $\delta_{\mathrm{cos}} \in [0.9,0.99]$ (step $0.01$). In \texttt{Jackal} environment, Q{\footnotesize OED} achieves an average dynamics prediction RMSE of \num{2.42\pm0.11}, while Q{\footnotesize OED}-A{\footnotesize GNOSTIC} yields \num{4.77\pm2.51}. These results indicate that Q{\footnotesize OED} is robust to the choice of hyperparameters over a wide range. While better settings may exist for specific tasks, we use our defaults because they are simple and perform well.

\subsection{Does Q{\footnotesize OED} help policy learning?}
\label{sec:sim-policy}

\textbf{Baselines and Ablations.}
We compare against the following strong exploration baselines. \textbf{(1) S{\footnotesize AC}-A{\footnotesize DAPT}} (SAC~\cite{haarnoja2018soft}) adaptively tunes the weights of task and exploration rewards~\cite{chen22eipo} to balance exploration and performance.
\textbf{(2) D{\footnotesize ISAGREEMENT}} (SAC) follows an explore-then-exploit schedule: it maximizes an exploration reward for the first \num{25}\% of interactions, then switches to maximizing task reward~\cite{pathak19disagreement,sekar2020planning,NEURIPS2021_cc4af25f}. We use ensemble disagreement as the exploration reward, computed from an ensemble of five learned dynamics. \textbf{(3) D{\footnotesize OMAIN}-R{\footnotesize ANDOM}} trains PPO with domain randomization (DR)~\cite{chen2022understanding} using mjlab settings.
We also evaluate Q{\footnotesize OED} ablations: \textbf{(4) Q{\footnotesize OED}-P{\footnotesize HYSICS}} (finite differences with ground-truth simulator physics), \textbf{(5) Q{\footnotesize OED}-A{\footnotesize GNOSTIC}} (matching prior settings~\cite{memmel2024asid,sobanbabu2025sampling}), and \textbf{(6) B{\footnotesize OED}}~\cite{rainforth2024modern}. PPO and SAC follow Stable-Baselines3~\cite{stable-baselines3} (with adjustments in Appendix~\ref{apdx:model-struc}). For fairness, all methods except Q{\footnotesize OED}-P{\footnotesize HYSICS} use our learned dynamics for policy optimization.

\textbf{Results.}
Fig.~\ref{fig:sim-policy} answers \textbf{Q2} by reporting episodic \emph{task} reward (excluding exploration reward), averaged over ten seeds. Results are normalized to $[0,1]$ per environment for comparability. Q{\footnotesize OED}-P{\footnotesize HYSICS} performs best across all environments since it has access to ground-truth physics, confirming that when the dynamics model is accurate, Q{\footnotesize OED} can achieve strong policy performance. In contrast, B{\footnotesize OED} is greedy for information, leading to weaker
policies. Q{\footnotesize OED}-A{\footnotesize GNOSTIC} focuses on critical parameters but underperforms Q{\footnotesize OED} because it ignores the influence of dropped parameters; as in the previous experiment, poorer dynamics estimates yield miscalibrated physics during policy optimization. A similar issue appears in D{\footnotesize OMAIN}-R{\footnotesize ANDOM}, where broad domain-randomization priors do not efficiently improve online performance. D{\footnotesize ISAGREEMENT} is a strong baseline by targeting regions of high dynamics-model uncertainty, an approach that is effective in model-based RL (e.g., PETS~\cite{chua2018pets}) compared to SAC baselines. Nonetheless, in our robotics scenarios, conditioning the dynamics model on physics parameters yields further gains; see Scaffolder~\cite{hu2024privileged} for a broader discussion of incorporating privileged information in dynamics models. Q{\footnotesize OED} performance drops in \texttt{G1-Rough} because the learned dynamics model does not reliably capture the critical parameters during early learning. Appendix~\ref{apdx:exp} provides further result analysis of the learned dynamics model.

\section{Real-World Experiment}
We evaluate \textbf{Q1} and \textbf{Q2} in real-world experiments on two platforms: a Franka Emika Panda arm (manipulation) and a Clearpath Jackal mobile vehicle (navigation).

\subsection{Rod balance}

\begin{figure}[t!]
\vspace*{0.08in}
\centering
\begin{subfigure}[h]{0.241\textwidth}
\centering
    \includegraphics[width=0.985\textwidth, trim={10cm 0cm 10cm 0cm}, clip]{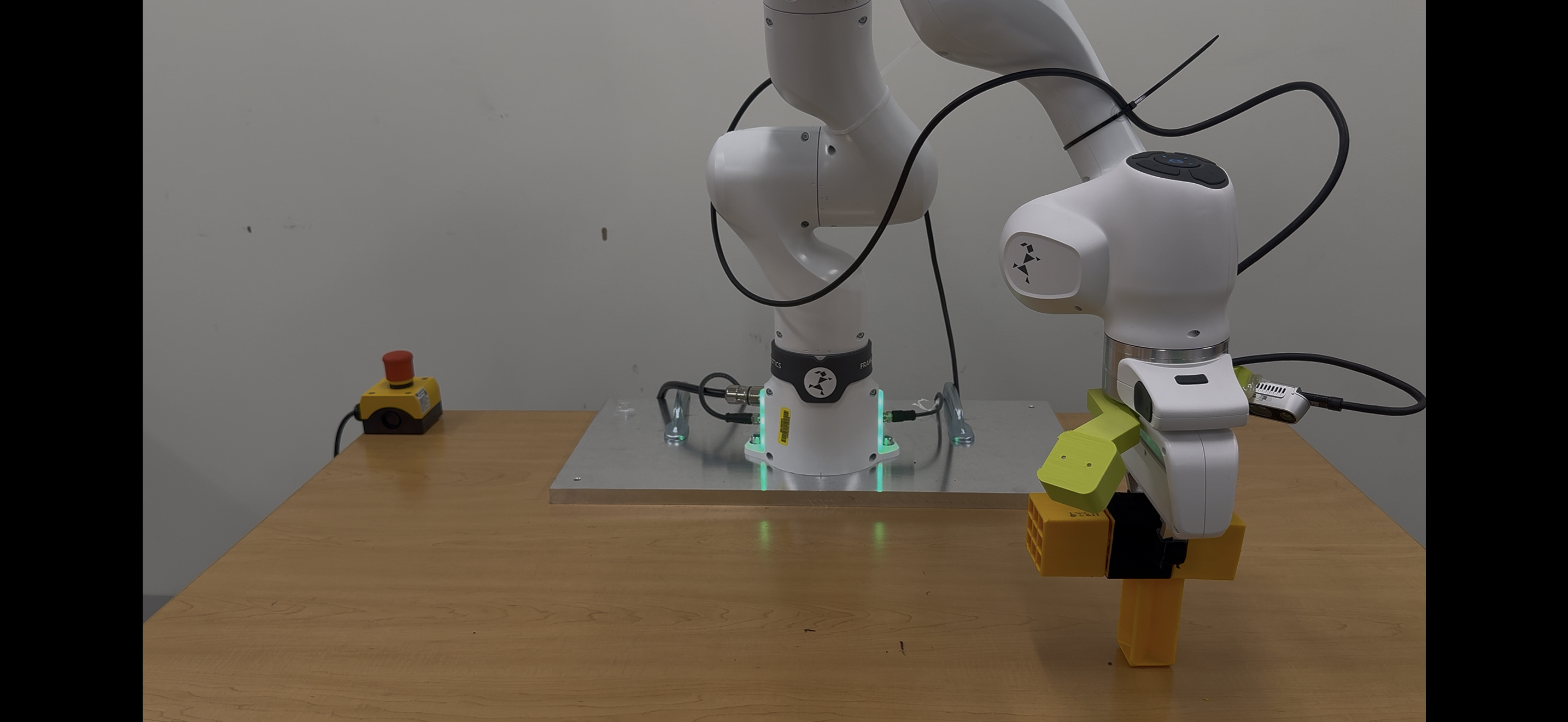}
    \label{fig:franka-cube1}
\end{subfigure}
\begin{subfigure}[h]{0.241\textwidth}
\centering
    \includegraphics[width=0.985\textwidth, trim={10cm 0cm 10cm 0cm}, clip]{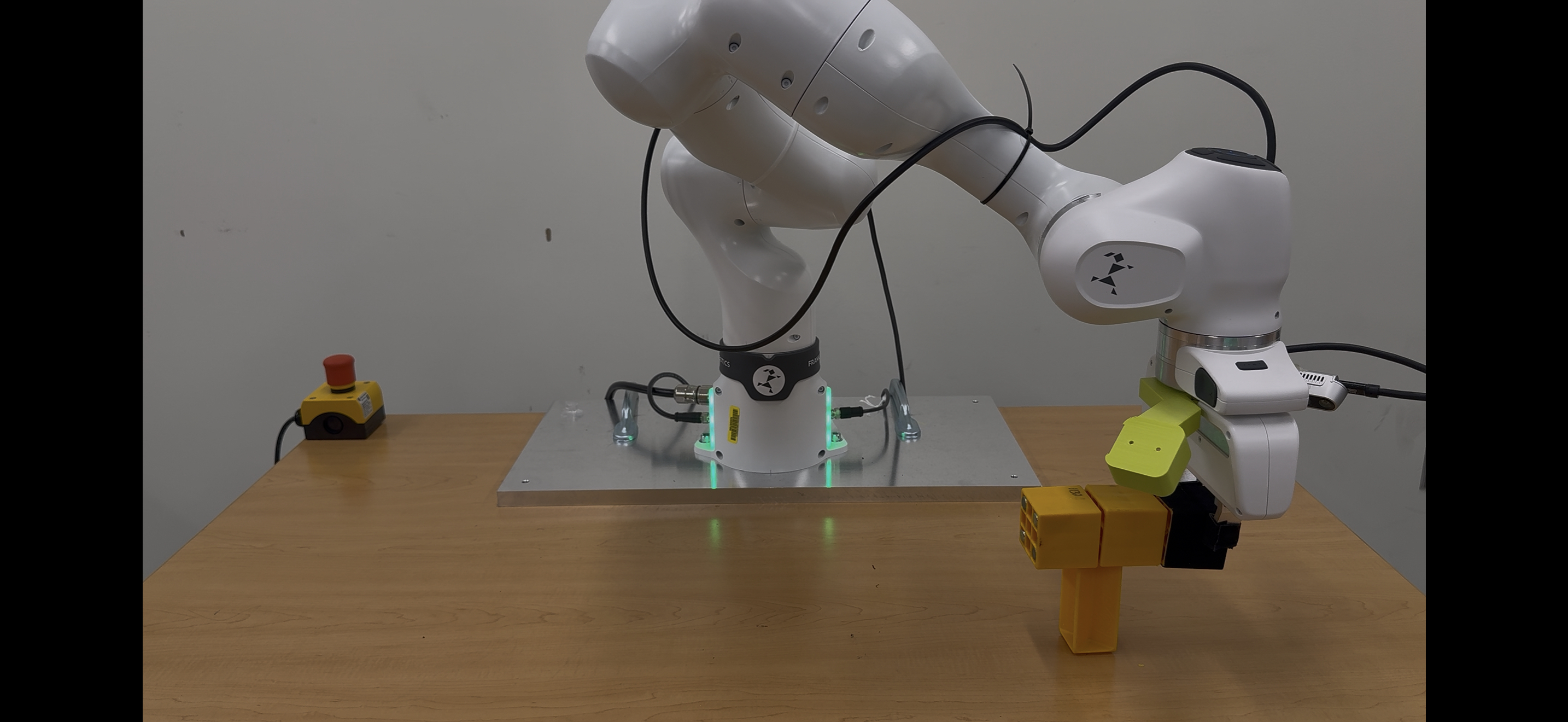}
    \label{fig:franka-cube2}
\end{subfigure}
\centering
\begin{subfigure}[h]{0.5\textwidth}
\centering
    \vspace*{-0.05in}
    \includegraphics[width=0.98\textwidth]{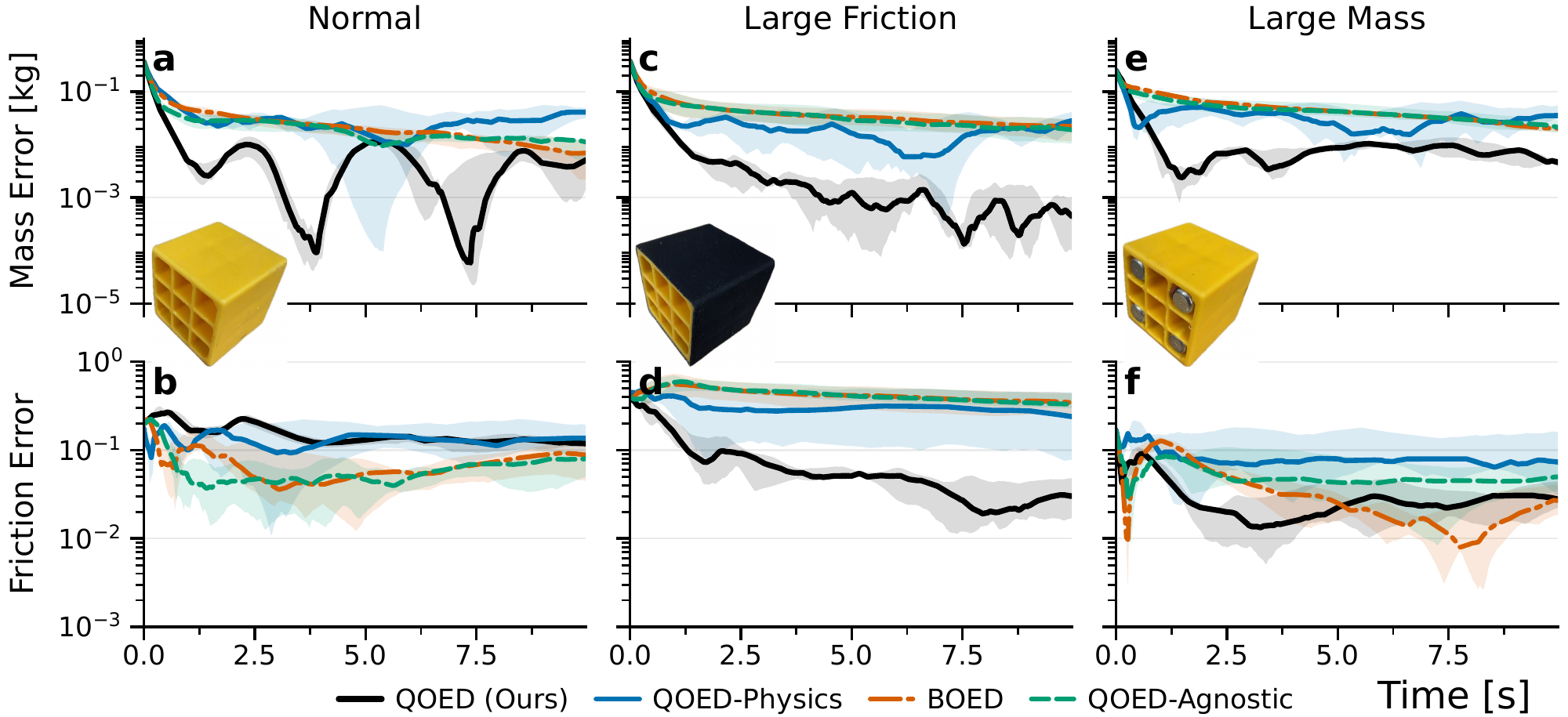}
    \label{fig:franka-result}
\end{subfigure}
\caption{\small Rod balancing demonstration and parameter-estimation error bars. Our Q{\footnotesize OED} identifies the parameters quickly and accurately.}
\label{fig:franka-param-err}
\end{figure}

This section primarily answers \textbf{Q1}, since the Franka setup provides high-accuracy reference measurements for comparison. We use the same ablations as in Sect.~\ref{sec:sim-policy}. Note that Q{\footnotesize OED}-P{\footnotesize HYSICS} updates the real-to-sim residual distribution using an unscented Kalman filter~\cite{Schperberg23UR}.

\textbf{Environment.}
We use a Franka Emika Panda equipped with a RealSense D435i camera (\SI{30}{\Hz}) and an NVIDIA 4090 GPU.
The robot must identify the mass, inertia, and friction of three cubes, then pick the cube with the highest friction and balance the stacked cubes
(Fig.~\ref{fig:franka-param-err}). We create three task variants by randomly stacking the cubes. For each variant and each method, we run six trials.

\textbf{Results.}
Despite the apparent simplicity of the task, the ablations often fail. Q{\footnotesize OED} achieves an $\bm{89\%}$ success rate, whereas
Q{\footnotesize OED}-P{\footnotesize HYSICS}, B{\footnotesize OED}, and Q{\footnotesize OED}-A{\footnotesize GNOSTIC} achieve $\bm{42\%}$, $\bm{8\%}$,
and $\bm{17\%}$, respectively. Failures are typically caused by center-of-mass misalignment: an error of only \SI{2}{\centi\meter} is sufficient to
topple the stack, contributing to the low success rate of Q{\footnotesize OED}-P{\footnotesize HYSICS}. During exploration, Q{\footnotesize OED}
adopts a simple objective-adaptation strategy---first identifying mass and inertia, then estimating friction---leading to faster parameter convergence
(Fig.~\ref{fig:franka-param-err}). The parameter estimation curve is not smoothed, as we opt to present the raw results rather than apply the filters. Although Q{\footnotesize OED}-A{\footnotesize GNOSTIC} targets a similar objective-adaptation strategy, it does not account for the
influence of friction when estimating mass, making the resulting exploration behavior (e.g., hefting the cube) less reliable. Even with the Franka
arm's accurate sensing, jointly identifying mass/inertia and friction is challenging, which helps explain the poor performance of
Q{\footnotesize OED}-A{\footnotesize GNOSTIC} and B{\footnotesize OED}.

\subsection{Wild Wheeled Navigation}
\begin{figure}[t!]
\vspace*{0.08in}
\centering
\begin{subfigure}[h]{0.5\textwidth}
\centering
    \includegraphics[width=0.985\textwidth, trim={0.35cm 3.1cm 0.35cm 3.1cm}, clip]{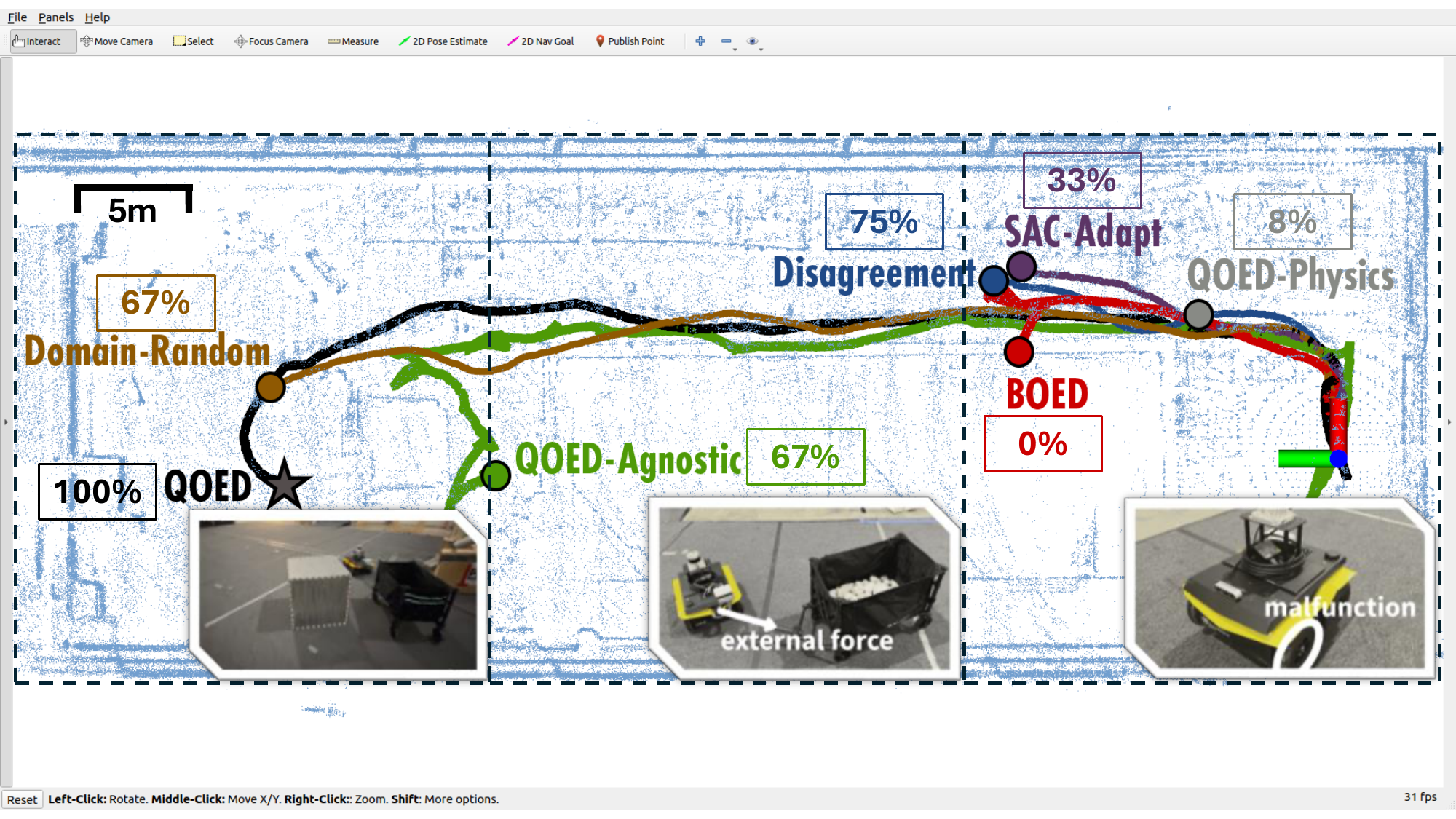}
    \label{fig:jackal-mesh}
\end{subfigure}
\begin{subfigure}[h]{0.5\textwidth}
\centering
    \vspace*{0.05in}
    \includegraphics[width=0.985\textwidth, trim={0.35cm 3.1cm 0.35cm 3cm}, clip]{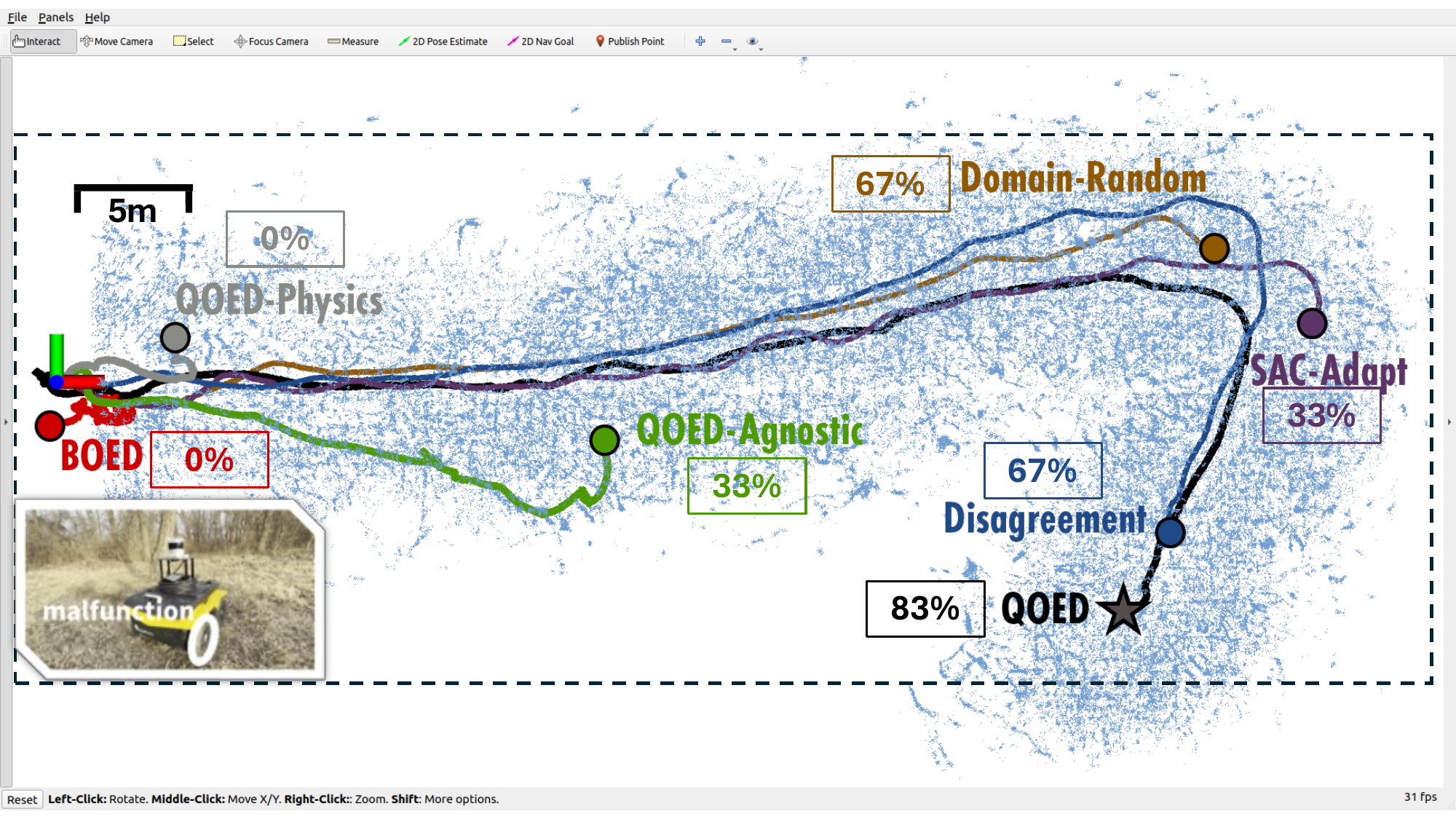}
    \label{fig:jackal-forest}
\end{subfigure}
\begin{subfigure}[h]{0.5\textwidth}
\centering
    \vspace*{0.1in}
    \includegraphics[width=1\textwidth]{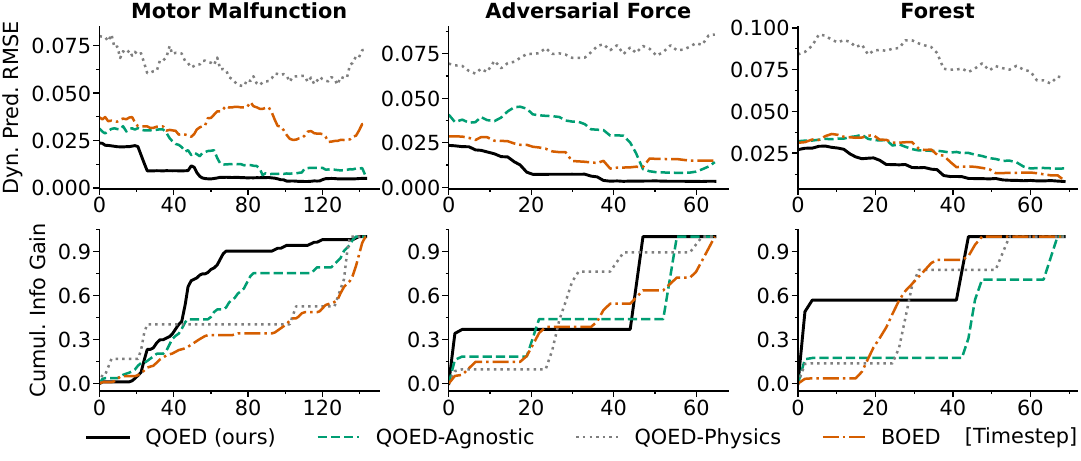}
    \label{fig:jackal-mesh-fim}
\end{subfigure}
\vspace*{-0.1in}
\caption{\small Real-world snapshots with success rates shown in the text boxes. Q{\footnotesize OED} achieves the highest success rate and the lowest dynamics prediction RMSE. By explicitly suppressing nuisance directions, it attains the highest cumulative information gain across environments.}
\label{fig:jackal-mesh-snapshot}
\end{figure}

This section further evaluates \textbf{Q1} and \textbf{Q2} in real-world navigation. Since ground-truth parameters are unavailable, we use the
dynamics prediction RMSE as the primary indicator. We use the same baselines and ablations as in Sect.~\ref{sec:sim-policy}.

\textbf{Environment.}
We use a Clearpath Jackal equipped with an OS1-64 LiDAR for LiDAR--inertial odometry~\cite{fasterlio,yu2023mimu} at \SI{100}{\Hz} and a RealSense D435i camera
for navigation at \SI{30}{\Hz}; all computation runs on a Jetson Orin SoC. We evaluate three scenarios and a transition:
\texttt{Motor Malfunction} (front-left wheel axle broken), \texttt{Adversarial Force} (towing a cart loaded with rocks), and \texttt{Forest} under \texttt{Motor Malfunction}. The robot must reach a goal while avoiding sparse obstacles. After \texttt{Adversarial Force}, we unload the trailer and move it to the \texttt{Forest} environment. For each scenario, we run six trials.

\textbf{Results.}
Fig.~\ref{fig:jackal-mesh-snapshot} reports real-world trajectories together with success rate, dynamics prediction RMSE, and normalized cumulative information gain. Q{\footnotesize OED} reaches the goal by quickly identifying key parameters, whereas baselines often crash due to miscalibrated physical coefficients. It also achieves the highest cumulative information gain across scenarios. Q{\footnotesize OED} first identifies mass and friction, then diagnoses wheel-related effects and longitudinal force while suppressing other coupled directions, yielding faster and more stable estimation than Q{\footnotesize OED}-A{\footnotesize GNOSTIC}; B{\footnotesize OED} estimation often fails to converge. In the recorded data, Q{\footnotesize OED} spends \SI{52.5}{\percent} of exploration time on mass and friction, whereas Q{\footnotesize OED}-A{\footnotesize GNOSTIC} and B{\footnotesize OED} spend \SI{87.5}{\percent} and \SI{89.5}{\percent}, respectively, which slows exploration and degrades estimation of motor and adversarial-force effects. In the real world, the learned dynamics model further improves performance over Q{\footnotesize OED}-P{\footnotesize HYSICS}, benefiting from the data efficiency of MBRL. D{\footnotesize ISAGREEMENT} is also strong due to its ensemble-based uncertainty signal, whereas S{\footnotesize AC}-A{\footnotesize DAPT} exhibits wobbling behaviors that are uninformative for policy learning and can lead to obstacle collisions. See Appendix~\ref{apdx:exp} for additional discussions.

\subsection{Discussions and Limitations.}
Although the SAC-based baselines perform poorly in our scenarios, they are usually more general than our MBPO setup because we require physics parameterization to define the exploration objective. One way to relax this requirement is to learn a latent parameterization (e.g., with a variational autoencoder) and perform exploration in the latent space. In the real world, limited training time also leads to less refined behavior than in simulation: BOED-style objectives can be difficult to learn with RL, as observed in prior work~\cite{memmel2024asid,sobanbabu2025sampling}. This creates a tension between objective complexity and exploration efficiency, which remains an important direction for future work.

\section{Conclusion and Future Work}
\label{sec:conclusion}
We presented \emph{Quasi-Optimal Experimental Design} (Q{\footnotesize OED}), an adaptive information objective for robot exploration.
Grounded in optimal experimental design, Q{\footnotesize OED} identifies critical parameters online and prioritizes them while suppressing nuisance
directions. Extensive simulation and real-world experiments show that Q{\footnotesize OED} collects informative data and efficiently reduces model error. We also demonstrate benefits in model-based policy optimization with a learned dynamics model, allowing Q{\footnotesize OED} to outperform purely physics-based analytical models. Finally, we plan to extend Q{\footnotesize OED} to broader robot learning problems, such as meta-learning RL update rules by treating them as parameters of interest, or optimizing the hyperparameters of differentiable algorithms, including differentiable MPC.

% \section*{Acknowledgments}

\newpage
%% Use plainnat to work nicely with natbib. 
% \bibliographystyle{plainnat}
\bibliographystyle{unsrtnat}
\bibliography{references}

\newpage

\clearpage
\appendix

\subsection{Shortcut Flow Matching}
\label{apdx:fmpe}
In the main text, we instantiate the surrogate transition likelihood $q_{\bm{\theta}}$ using shortcut models~\cite{frans2025one}. This appendix provides the full training objective and shows how it induces both (i) a one-step transport map for sampling and (ii) a differentiable conditional log-density.

\subsubsection{Setup and notation}
For each transition $(\mathbf{s}_t,\bm{a}_t,\bm{\phi},\mathbf{s}_{t+1})$ in the dataset,
we define the conditioning variable
\begin{equation*}
\label{eq:apdx-conditioning}
    \mathbf{c} = (\mathbf{s}_t,\bm{a}_t,\bm{\phi}).
\end{equation*}

To match the dynamics model definition in the main text, we model the state increment
\begin{equation*}
\label{eq:apdx-increment}
    \mathbf{x} = \mathbf{s}_{t+1}-\mathbf{s}_t
\end{equation*}

Let the base noise distribution be $p_0(\bm{\delta})=\mathcal{N}(\bm{0},\bm{I})$.
Given a data increment $\mathbf{x}$, a noise sample $\bm{\delta}$, and a ``time'' $u\in[0,1]$, we define the linear interpolation
\begin{equation*}
\label{eq:apdx-linear-path}
    \mathbf{z}_u
    :=
    (1-u)\,\mathbf{x}
    +
    u\,\bm{\delta}.
\end{equation*}
Along this path, the target velocity is known:
\begin{equation*}
\label{eq:apdx-target-velocity}
    \mathbf{v}
    =
    \frac{\mathrm{d}}{\mathrm{d} u}\mathbf{z}_u
    =
    \bm{\delta}-\mathbf{x}.
\end{equation*}

\subsubsection{Learning objective}
Shortcut models learn a conditional velocity field
$\mathbf{v}_{\bm{\theta}}(\mathbf{z}_u,u,\mathbf{c},d)$,
where $d\ge 0$ is an extra input that represents a step size.
The total loss is
\begin{equation*}
\label{eq:apdx-total-loss}
    \mathcal{L}(\bm{\theta})
    :=
    \mathcal{L}_{\mathrm{FM}}(\bm{\theta})
    +
    \mathcal{L}_{\mathrm{SC}}(\bm{\theta}).
\end{equation*}

The flow-matching term $\mathcal{L}_{\mathrm{FM}}$ matches the predicted velocity to the target velocity:
\begin{equation*}
\label{eq:apdx-lfm}
    \mathcal{L}_{\mathrm{FM}}(\bm{\theta})
    =
    \mathbb{E}_{(\mathbf{x},\mathbf{c})\sim\mathcal{D},\bm{\delta}\sim p_0,u\sim\mathcal{U}[0,1]}
    \Big[
        \big\|
            \mathbf{v}_{\bm{\theta}}(\mathbf{z}_u,u,\mathbf{c},0)
            -
            (\bm{\delta}-\mathbf{x})
        \big\|^2
    \Big].
\end{equation*}

The self-consistency term $\mathcal{L}_{\mathrm{SC}}$ enforces that one step of size $2d$ agrees with two sequential steps of size $d$.
The target velocity is defined as
\begin{equation*}
\label{eq:apdx-vtarget}
\begin{aligned}
    \mathbf{v}_{\mathrm{tgt}}
    &=
    \frac{1}{2}\Big[
        \mathbf{v}_{\bm{\theta}}(\mathbf{z}_u,u,\mathbf{c},d)
        \\& \qquad +
        \mathbf{v}_{\bm{\theta}}\!\Big(
            \mathbf{z}_u - d\cdot \mathbf{v}_{\bm{\theta}}(\mathbf{z}_u,u,\mathbf{c},d),
            \;u-d,\;\mathbf{c},\;d
        \Big)
    \Big].
\end{aligned}
\end{equation*}

To keep all arguments valid (in particular $u-d\in[0,1]$ and the implied two-step endpoint $u-2d\ge 0$),
we sample $d$ such that $0\le d\le u/2$.
The self-consistency loss is
\begin{equation*}
\label{eq:apdx-lsc}
\begin{aligned}
    \mathcal{L}_{\mathrm{SC}}(\bm{\theta})
    &=
    \mathbb{E}_{(\mathbf{x},\mathbf{c})\sim\mathcal{D},\bm{\delta}\sim p_0,\;u\sim\mathcal{U}[0,1],\;d\sim\mathcal{U}[0,u/2]}
    \\& \qquad \Big[
        \big\|
            \mathbf{v}_{\bm{\theta}}(\mathbf{z}_u,u,\mathbf{c},2d)
            -
            \mathbf{v}_{\mathrm{tgt}}
        \big\|^2
    \Big].
\end{aligned}
\end{equation*}

\subsubsection{One-step sampling}
After training, we can generate an increment in a single step.
Given $\bm{\delta}\sim p_0(\cdot)$ and conditioning $\mathbf{c}$, define the one-step map
\begin{equation*}
\label{eq:apdx-one-step-map}
    T_{\bm{\theta}}(\bm{\delta},\mathbf{c})
    :=
    \bm{\delta}
    -
    \mathbf{v}_{\bm{\theta}}(\bm{\delta},1,\mathbf{c},1).
\end{equation*}

We sample an increment $\hat{\mathbf{x}} = T_{\bm{\theta}}(\bm{\delta},\mathbf{c})$
and then set $\mathbf{s}_{t+1}=\mathbf{s}_t+\hat{\mathbf{x}}$.
In the main text, we overload notation and write this transport map as $T_{\bm{\theta}}(\bm{\delta},\mathbf{s}_t,\bm{a}_t,\bm{\phi})$.

\subsubsection{Conditional log-density}
If $T_{\bm{\theta}}(\cdot,\mathbf{c})$ is a diffeomorphism in $\bm{\delta}$,
the conditional density follows from the change-of-variables formula:
\begin{equation*}
\label{eq:apdx-cov}
    \log q_{\bm{\theta}}(\mathbf{x}\vert\mathbf{c})
    =
    \log p_0(\bm{\delta})
    -
    \log\left|
        \det
        \nabla_{\bm{\delta}}
        T_{\bm{\theta}}(\bm{\delta},\mathbf{c})
    \right|,
\end{equation*}
with $\bm{\delta}=T_{\bm{\theta}}^{-1}(\mathbf{x},\mathbf{c})$.
For the shortcut-model transport map, the Jacobian is
\begin{equation*}
\label{eq:apdx-jac}
    \nabla_{\bm{\delta}} T_{\bm{\theta}}(\bm{\delta},\mathbf{c})
    =
    \mathbf{I}
    -
    \nabla_{\bm{\delta}}
    \mathbf{v}_{\bm{\theta}}(\bm{\delta},1,\mathbf{c},1),
\end{equation*}
so the conditional log-density becomes
\begin{equation*}
\label{eq:apdx-loglik-final}
    \log q_{\bm{\theta}}(\mathbf{x}\vert\mathbf{c})
    =
    \log p_0(\bm{\delta})
    -
    \log \left|
        \det\!\left(
            \mathbf{I}
            -
            \nabla_{\bm{\delta}}
            \mathbf{v}_{\bm{\theta}}(\bm{\delta},1,\mathbf{c},1)
        \right)
    \right|.
\end{equation*}

\subsection{Derivation of the Trajectory Log-Likelihood}
\label{apdx:deri-traj-log-likelihood}

We derive the score of the surrogate trajectory likelihood with respect to parameters $\bm{\phi}$.
Let a length-$t$ trajectory be
\begin{equation*}
\bm{\tau}_t
=
[\mathbf{s}_0,\bm{a}_0,\mathbf{s}_1,\dots,\bm{a}_{t-1},\mathbf{s}_t].
\end{equation*}

Under a fixed policy $\bm{\pi}(\bm{a}_k\vert \mathbf{s}_k)$ and learned transitions
$q_{\bm{\theta}}(\mathbf{s}_{k+1}\vert \mathbf{s}_k,\bm{a}_k,\bm{\phi})$,
the surrogate trajectory likelihood factorizes as
\begin{equation*}
\label{eq:apdx-traj-factor}
q_{\bm{\theta}}(\bm{\tau}_t \vert \bm{\phi},\bm{\pi})
=
p(\mathbf{s}_0)\;
\prod_{k=0}^{t-1}
\bm{\pi}(\bm{a}_k \vert \mathbf{s}_k)\;
q_{\bm{\theta}}(\mathbf{s}_{k+1} \vert \mathbf{s}_k, \bm{a}_k, \bm{\phi}).
\end{equation*}

Taking logs gives
\begin{equation*}
\begin{aligned}
\label{eq:apdx-traj-loglik}
\log q_{\bm{\theta}}(\bm{\tau}_t \vert \bm{\phi},\bm{\pi})
&=
\log p(\mathbf{s}_0)
+
\sum_{k=0}^{t-1}\log \bm{\pi}(\bm{a}_k \vert \mathbf{s}_k)
\\&+
\sum_{k=0}^{t-1}\log q_{\bm{\theta}}(\mathbf{s}_{k+1} \vert \mathbf{s}_k, \bm{a}_k, \bm{\phi}).
\end{aligned}
\end{equation*}

The first two terms do not depend on $\bm{\phi}$ (we treat the policy as fixed when computing information about $\bm{\phi}$).
Therefore,
\begin{equation*}
\label{eq:apdx-traj-score}
\nabla_{\bm{\phi}}
\log q_{\bm{\theta}}(\bm{\tau}_t \vert \bm{\phi},\bm{\pi})
=
\sum_{k=0}^{t-1}
\nabla_{\bm{\phi}}
\log q_{\bm{\theta}}(\mathbf{s}_{k+1} \vert \mathbf{s}_k, \bm{a}_k, \bm{\phi}).
\end{equation*}

To relate the per-step term to the shortcut-model construction in Appendix~\ref{apdx:fmpe},
define the state increment $\mathbf{x}_k = \mathbf{s}_{k+1}-\mathbf{s}_k$ and conditioning $\mathbf{c}_k = (\mathbf{s}_k,\bm{a}_k,\bm{\phi})$.
The one-step transport map produces $\mathbf{x}_k$ from $\bm{\delta}_k\sim p_0(\cdot)=\mathcal{N}(\bm{0},\bm{I})$ via
$\mathbf{x}_k = T_{\bm{\theta}}(\bm{\delta}_k,\mathbf{c}_k)$.
If $T_{\bm{\theta}}(\cdot,\mathbf{c}_k)$ is invertible, change of variables~\cite{chen2018neural} gives
\begin{equation*}
\begin{aligned}
\label{eq:apdx-step-loglik}
\log q_{\bm{\theta}}(\mathbf{s}_{k+1}\vert \mathbf{s}_k,\bm{a}_k,\bm{\phi})
&=
\log q_{\bm{\theta}}(\mathbf{x}_k\vert \mathbf{c}_k)
\nonumber\\
&=
\log p_0(\bm{\delta}_k)
-
\log\left|
\det
\nabla_{\bm{\delta}_k}
T_{\bm{\theta}}(\bm{\delta}_k,\mathbf{c}_k)
\right|,
\end{aligned}
\end{equation*}
where $\bm{\delta}_k = T_{\bm{\theta}}^{-1}(\mathbf{x}_k,\mathbf{c}_k)$.
In our implementation, the gradient is obtained by automatic differentiation of the per-step expression above.

\subsection{Proof of Lemma~\ref{lem:fim-direction}}
\label{apdx:proof-fim-direction-lemma}
\begin{proof}
Recall the Fisher information matrix at $\bm{\phi}$:
\begin{equation*}
\begin{aligned}
    \bm{\mathcal{F}}_{\bm{\phi}}
    &=
    \mathbb{E}_{\bm{\tau}\sim q_{\bm{\theta}}(\cdot\vert \bm{\phi}, \bm{\pi})}
    \Big[
        \mathbf{g}(\bm{\tau},\bm{\phi})\;
        \mathbf{g}(\bm{\tau},\bm{\phi})^{\top}
    \Big],
    \\
    \mathbf{g}(\bm{\tau},\bm{\phi})
    &=
    \nabla_{\bm{\phi}} \log q_{\bm{\theta}}(\bm{\tau}\vert \bm{\phi}, \bm{\pi}).
\end{aligned}
\end{equation*}

Each outer product $\mathbf{g}\mathbf{g}^{\top}$ is symmetric, so their expectation $\bm{\mathcal{F}}_{\bm{\phi}}$ is also symmetric. Next, for any vector $\mathbf{v}\in\mathbb{R}^m$,
\begin{equation*}
    \mathbf{v}^{\top}\bm{\mathcal{F}}_{\bm{\phi}}\mathbf{v}
    =
    \mathbb{E}\Big[
        \mathbf{v}^{\top}(\mathbf{g}\mathbf{g}^{\top})\mathbf{v}
    \Big]
    =
    \mathbb{E}\Big[
        (\mathbf{g}^{\top}\mathbf{v})^2
    \Big]
    \ge 0,
\end{equation*}
so $\bm{\mathcal{F}}_{\bm{\phi}}$ is positive semidefinite. Therefore, $\bm{\mathcal{F}}_{\bm{\phi}}$ admits an eigen-decomposition $\bm{\mathcal{F}}_{\bm{\phi}}=\mathbf{W}\mathbf{\Lambda} \mathbf{W}^{\top}$ with orthonormal eigenvectors $\mathbf{W}=[\mathbf{w}_1,\dots,\mathbf{w}_m]$ and eigenvalues $\mathbf{\Lambda}=\operatorname{diag}(\lambda_1,\dots,\lambda_m)$.

Finally, for any eigenvector $\mathbf{w}_i$ (with $\|\mathbf{w}_i\|=1$),
\begin{equation*}
    \lambda_i
    =
    \mathbf{w}_i^{\top}\bm{\mathcal{F}}_{\bm{\phi}} \mathbf{w}_i
    =
    \mathbb{E}_{\bm{\tau}\sim q_{\bm{\theta}}(\cdot\vert \bm{\phi})}
    \Big[
        \big(
            \mathbf{g}(\bm{\tau},\bm{\phi})^{\top} \mathbf{w}_i
        \big)^2
    \Big],
\end{equation*}
which is exactly the statement of Lemma~\ref{lem:fim-direction}.
\end{proof}

\subsection{Trace form of the agnostic objective}
\label{apdx:naive-bonus-trace}

Eq.~\eqref{eqn:naive-bonus} defines the agnostic objective as the trace of a \emph{coordinate-restricted} FIM.
Here we record equivalent expressions and relate this trace to the global eigendecomposition of the full FIM.

\begin{mdframed}[hidealllines=true,backgroundcolor=MaterialBlue10!40,innerleftmargin=.2cm,
  innerrightmargin=.2cm]
\begin{lemma}[Equivalent forms of $\tr(\bm{\mathcal{F}}_{\mathbf{k}\mathbf{k}})$]
\label{lem:naive-trace-forms}
Let $\bm{\mathcal{F}}\in\mathbb{R}^{m\times m}$ be the full FIM and let $\bm{\mathcal{F}}=\mathbf{W}\mathbf{\Lambda}\mathbf{W}^\top$
be an eigendecomposition with $\mathbf{\Lambda}=\mathrm{diag}(\lambda_1,\dots,\lambda_m)$.
For any coordinate index set $\mathbf{k}\subseteq\{1,\dots,m\}$, let $\mathbf{S}_{\mathbf{k}}\in\{0,1\}^{|\mathbf{k}|\times m}$ be the
row-selector matrix (so that $\mathbf{S}_{\mathbf{k}}\bm{\phi}=\bm{\phi}_{\mathbf{k}}$) and define the principal submatrix
$\bm{\mathcal{F}}_{\mathbf{k}\mathbf{k}}:=\mathbf{S}_{\mathbf{k}}\bm{\mathcal{F}}\mathbf{S}_{\mathbf{k}}^\top$.
Then
\begin{equation*}
\label{eq:naive-trace-global-eig}
\begin{aligned}
\tr(\bm{\mathcal{F}}_{\mathbf{k}\mathbf{k}})
&=
\tr\!\left(\mathbf{S}_{\mathbf{k}}\mathbf{W}\mathbf{\Lambda}\mathbf{W}^\top\mathbf{S}_{\mathbf{k}}^\top\right)
\\&=
\tr\!\left(\mathbf{\Lambda}\,\mathbf{W}_{\mathbf{k}}^\top\mathbf{W}_{\mathbf{k}}\right)
=
\sum_{i=1}^m \lambda_i \,\|\mathbf{W}_{\mathbf{k}}[:,i]\|_2^2,
\end{aligned}
\end{equation*}
where $\mathbf{W}_{\mathbf{k}}:=\mathbf{S}_{\mathbf{k}}\mathbf{W}$ collects the rows of $\mathbf{W}$ indexed by $\mathbf{k}$.

Moreover, if $\bm{\mathcal{F}}_{\mathbf{k}\mathbf{k}}=\mathbf{U}\mathbf{\Lambda}_{\mathbf{k}}\mathbf{U}^\top$ is the eigendecomposition of the
principal submatrix (with $\mathbf{\Lambda}_{\mathbf{k}}=\mathrm{diag}(\tilde{\lambda}_1,\dots,\tilde{\lambda}_{|\mathbf{k}|})$), then
\begin{equation*}
\label{eq:naive-trace-sub-eig}
\tr(\bm{\mathcal{F}}_{\mathbf{k}\mathbf{k}})=\tr(\mathbf{\Lambda}_{\mathbf{k}})=\sum_{j=1}^{|\mathbf{k}|}\tilde{\lambda}_j.
\end{equation*}
\end{lemma}
\end{mdframed}

\begin{proof}
By definition, $\bm{\mathcal{F}}_{\mathbf{k}\mathbf{k}}=\mathbf{S}_{\mathbf{k}}\bm{\mathcal{F}}\mathbf{S}_{\mathbf{k}}^\top$.
Substituting $\bm{\mathcal{F}}=\mathbf{W}\mathbf{\Lambda}\mathbf{W}^\top$ and using cyclicity of the trace gives
\begin{equation*}
\begin{aligned}
\tr(\bm{\mathcal{F}}_{\mathbf{k}\mathbf{k}})
&=
\tr\!\left(\mathbf{S}_{\mathbf{k}}\mathbf{W}\mathbf{\Lambda}\mathbf{W}^\top\mathbf{S}_{\mathbf{k}}^\top\right)
\\&=
\tr\!\left(\mathbf{\Lambda}\,\mathbf{W}^\top\mathbf{S}_{\mathbf{k}}^\top\mathbf{S}_{\mathbf{k}}\mathbf{W}\right)
=
\tr\!\left(\mathbf{\Lambda}\,\mathbf{W}_{\mathbf{k}}^\top\mathbf{W}_{\mathbf{k}}\right),
\end{aligned}
\end{equation*}

The first equation follows from cyclicity of the trace and the definition $\mathbf{W}_{\mathbf{k}}=\mathbf{S}_{\mathbf{k}}\mathbf{W}$. The second equation follows from the eigendecomposition
$\bm{\mathcal{F}}_{\mathbf{k}\mathbf{k}}=\mathbf{U}\mathbf{\Lambda}_{\mathbf{k}}\mathbf{U}^\top$ and
$\tr(\mathbf{U}\mathbf{\Lambda}_{\mathbf{k}}\mathbf{U}^\top)=\tr(\mathbf{\Lambda}_{\mathbf{k}})$.
\end{proof}

\begin{remark}[Coordinate restriction vs.\ eigen-direction restriction]
The above equation involves eigenvalues of the \emph{principal submatrix} $\bm{\mathcal{F}}_{\mathbf{k}\mathbf{k}}$ and is not, in general,
equal to $\sum_{i\in \mathbf{k}}\lambda_i$ (a subset of eigenvalues of the full FIM), unless the leading eigenspaces are aligned with the coordinate axes.
The latter corresponds to \emph{eigen-subspace} selection, discussed separately in Appendix~\ref{apdx:proof-score-subspace}.
\end{remark}

\subsection{Proof of Theorem~\ref{thm:qoed-quasiopt-trace}}
\label{apdx:proof-quasi-optimal}
\begin{proof}
Make the policy dependence explicit by writing
$\bm{\mathcal{F}}^{\bm{\pi}} := \bm{\mathcal{F}}_{\bm{\phi}}$ in Eq.~\eqref{eqn:fim}. Assuming $\bm{\mathcal{F}}^{\bm{\pi}}_{\mathbf{k}\mathbf{k}}\succ 0$ and
$\bm{\mathcal{F}}^{\bm{\pi}}_{\overline{\mathbf{k}}\overline{\mathbf{k}}}\succ 0$, write the block partition
\begin{equation*}
\bm{\mathcal{F}}^{\bm{\pi}}
=
\begin{bmatrix}
\bm{\mathcal{F}}_{\mathbf{k}\mathbf{k}}^{\bm{\pi}} & \bm{\mathcal{F}}_{\mathbf{k}\overline{\mathbf{k}}}^{\bm{\pi}}\\
\bm{\mathcal{F}}_{\overline{\mathbf{k}}\mathbf{k}}^{\bm{\pi}} & \bm{\mathcal{F}}_{\overline{\mathbf{k}}\overline{\mathbf{k}}}^{\bm{\pi}}
\end{bmatrix},
\quad
\bm{\mathcal{I}}_{\mathbf{k}\mid\overline{\mathbf{k}}}^{\bm{\pi}}
=
\bm{\mathcal{F}}_{\mathbf{k}\mathbf{k}}^{\bm{\pi}}
-
\bm{\mathcal{F}}_{\mathbf{k}\overline{\mathbf{k}}}^{\bm{\pi}}
\left(\bm{\mathcal{F}}_{\overline{\mathbf{k}}\overline{\mathbf{k}}}^{\bm{\pi}}\right)^{-1}
\bm{\mathcal{F}}_{\overline{\mathbf{k}}\mathbf{k}}^{\bm{\pi}}.
\end{equation*}

Define the (nonnegative) confounding penalty
\begin{equation*}
\label{eq:confounding-penalty}
C(\bm{\pi})
:=
\tr\!\Big(
\bm{\mathcal{F}}_{\mathbf{k}\overline{\mathbf{k}}}^{\bm{\pi}}
\left(\bm{\mathcal{F}}_{\overline{\mathbf{k}}\overline{\mathbf{k}}}^{\bm{\pi}}\right)^{-1}
\bm{\mathcal{F}}_{\overline{\mathbf{k}}\mathbf{k}}^{\bm{\pi}}
\Big)
\;\ge\;0.
\end{equation*}
so that
\begin{equation*}
\label{eq:qobj-expand}
\mathcal{B}_{\text{QOED}}(\bm{\pi})
=
\tr\!\big(\bm{\mathcal{I}}_{\mathbf{k}\mid\overline{\mathbf{k}}}^{\bm{\pi}}\big)
=
\tr\!\big(\bm{\mathcal{F}}_{\mathbf{k}\mathbf{k}}^{\bm{\pi}}\big)
-
C(\bm{\pi}).
\end{equation*}

First, we bound $C(\bm{\pi})$ by $\beta\,\tr(\bm{\mathcal{F}}_{\mathbf{k}\mathbf{k}}^{\bm{\pi}})$. Let
\begin{equation*}
\mathbf{M}
:=
\big(\bm{\mathcal{F}}_{\mathbf{k}\mathbf{k}}^{\bm{\pi}}\big)^{-1/2}
\bm{\mathcal{F}}_{\mathbf{k}\overline{\mathbf{k}}}^{\bm{\pi}}
\big(\bm{\mathcal{F}}_{\overline{\mathbf{k}}\overline{\mathbf{k}}}^{\bm{\pi}}\big)^{-1/2}.
\end{equation*}
Then
\begin{equation*}
\begin{aligned}
C(\bm{\pi})
&=
\tr\!\Big(
\bm{\mathcal{F}}_{\mathbf{k}\overline{\mathbf{k}}}^{\bm{\pi}}
\left(\bm{\mathcal{F}}_{\overline{\mathbf{k}}\overline{\mathbf{k}}}^{\bm{\pi}}\right)^{-1}
\bm{\mathcal{F}}_{\overline{\mathbf{k}}\mathbf{k}}^{\bm{\pi}}
\Big) \\
&=
\tr\!\Big(
\left(\bm{\mathcal{F}}_{\mathbf{k}\mathbf{k}}^{\bm{\pi}}\right)^{1/2}\,
\mathbf{M}\mathbf{M}^\top\,
\left(\bm{\mathcal{F}}_{\mathbf{k}\mathbf{k}}^{\bm{\pi}}\right)^{1/2}
\Big)
=
\tr\!\Big(
\bm{\mathcal{F}}_{\mathbf{k}\mathbf{k}}^{\bm{\pi}}\,
\mathbf{M}\mathbf{M}^\top
\Big).
\end{aligned}
\end{equation*}

Since $\bm{\mathcal{F}}_{\mathbf{k}\mathbf{k}}^{\bm{\pi}}\succeq 0$ and $\mathbf{M}\mathbf{M}^\top\succeq 0$, we may use
$\tr(\mathbf{A}\mathbf{B})\le \|\mathbf{B}\|_2\,\tr(\mathbf{A})$ for positive semidefinite (PSD) $\mathbf{A},\mathbf{B}$ to obtain
\begin{equation*}
C(\bm{\pi})
\le
\|\mathbf{M}\mathbf{M}^\top\|_2\;\tr(\bm{\mathcal{F}}_{\mathbf{k}\mathbf{k}}^{\bm{\pi}})
=
\|\mathbf{M}\|_2^2\;\tr(\bm{\mathcal{F}}_{\mathbf{k}\mathbf{k}}^{\bm{\pi}}).
\end{equation*}

By definition of $\beta$ in Eq.~\eqref{eq:eta-beta-main}, we have $\|\mathbf{M}\|_2^2 \le \beta$, hence
\begin{equation*}
\label{eq:qobj-lower}
\mathcal{B}_{\text{QOED}}(\bm{\pi})
=
\tr(\bm{\mathcal{F}}_{\mathbf{k}\mathbf{k}}^{\bm{\pi}})-C(\bm{\pi})
\ge
(1-\beta)\,\tr(\bm{\mathcal{F}}_{\mathbf{k}\mathbf{k}}^{\bm{\pi}}).
\end{equation*}

Second, we relate $\tr(\bm{\mathcal{F}}_{\mathbf{k}\mathbf{k}})$ to $\tr(\bm{\mathcal{F}})$. By block additivity of the trace,
\begin{equation*}
\tr(\bm{\mathcal{F}}^{\bm{\pi}})
=
\tr(\bm{\mathcal{F}}_{\mathbf{k}\mathbf{k}}^{\bm{\pi}})
+
\tr(\bm{\mathcal{F}}_{\overline{\mathbf{k}}\overline{\mathbf{k}}}^{\bm{\pi}})
\le
(1+\eta)\,\tr(\bm{\mathcal{F}}_{\mathbf{k}\mathbf{k}}^{\bm{\pi}}),
\end{equation*}
where the inequality uses the definition of $\eta$. Therefore,
\begin{equation*}
\label{eq:kk-lower-full}
\tr(\bm{\mathcal{F}}_{\mathbf{k}\mathbf{k}}^{\bm{\pi}})
\ge
\frac{1}{1+\eta}\,\tr(\bm{\mathcal{F}}^{\bm{\pi}}).
\end{equation*}

Third, let $\bm{\pi}^\star\in\argmax_{\bm{\pi}}\tr(\bm{\mathcal{F}}^{\bm{\pi}})$ and
$\widehat{\bm{\pi}}\in\argmax_{\bm{\pi}}\tr(\bm{\mathcal{I}}^{\bm{\pi}}_{\mathbf{k}\mid\overline{\mathbf{k}}})$. Since $C(\bm{\pi})\ge 0$,
\begin{equation*}
\label{eq:trace-ge-qoed}
\tr\!\big(\bm{\mathcal{F}}^{\widehat{\bm{\pi}}}\big)
\ge
\tr\!\big(\bm{\mathcal{F}}^{\widehat{\bm{\pi}}}_{\mathbf{k}\mathbf{k}}\big)
\ge
\tr\!\big(\bm{\mathcal{I}}^{\widehat{\bm{\pi}}}_{\mathbf{k}\mid\overline{\mathbf{k}}}\big)
=
\mathcal{B}_{\text{QOED}}(\widehat{\bm{\pi}}).
\end{equation*}

Optimality of $\widehat{\bm{\pi}}$ for $\mathcal{B}_{\text{QOED}}$ implies
$\mathcal{B}_{\text{QOED}}(\widehat{\bm{\pi}})\ge \mathcal{B}_{\text{QOED}}(\bm{\pi}^\star)$. Combining with previous results (applied at $\bm{\pi}^\star$) yields
\begin{equation*}
\begin{aligned}
\tr\!\big(\bm{\mathcal{F}}^{\widehat{\bm{\pi}}}\big)
&\ge
\mathcal{B}_{\text{QOED}}(\widehat{\bm{\pi}})
\ge
\mathcal{B}_{\text{QOED}}(\bm{\pi}^\star)\\
&\ge
(1-\beta)\,\tr\!\big(\bm{\mathcal{F}}^{\bm{\pi}^\star}_{\mathbf{k}\mathbf{k}}\big)
\ge
\frac{1-\beta}{1+\eta}\,\tr\!\big(\bm{\mathcal{F}}^{\bm{\pi}^\star}\big),
\end{aligned}
\end{equation*}
which is exactly Eq.~\eqref{eq:quasiopt-main}. Under invertible block assumptions, $\beta\le 1$ follows from the PSD block structure of the FIM, and $\beta<1$ excludes the degenerate case where critical score directions are perfectly linearly predictable from nuisance directions. Assumption~\ref{assump:regularity} implies $\tr(\bm{\mathcal{F}}^{\bm{\pi}})=\mathbb{E}\|\mathbf{g}\|_2^2$ is uniformly bounded over $\Pi$. Thus $\eta<\infty$ holds whenever the critical block remains non-degenerate over $\Pi$, e.g., $\inf_{\bm{\pi}\in\Pi}\tr(\bm{\mathcal{F}}^{\bm{\pi}}_{\mathbf{k}\mathbf{k}})>0$.

\end{proof}

$\eta$ measures the relative \emph{nuisance information mass} in $\overline{\mathbf{k}}$ compared to $\mathbf{k}$, and $\beta$ measures \emph{cross-block confounding} (how well the critical score can be linearly predicted from the nuisance score). Small $\eta$ and $\beta$ imply that maximizing Q{\footnotesize OED} also approximately maximizes the full BOED trace objective.

\begin{mdframed}[hidealllines=true,backgroundcolor=MaterialBlue10!40,innerleftmargin=.2cm,
  innerrightmargin=.2cm]
\begin{corollary}[Quasi-optimality with adaptive index selection]
\label{cor:adaptive-k}
Let $K:\Pi\to 2^{\{1,\dots,m\}}$ be any (possibly policy-dependent) rule and define
$\mathbf{k}=K(\pi)$ and $\overline{\mathbf{k}}:=\{1,\dots,m\}\setminus \mathbf{k}$.
Let
\begin{equation*}
\widetilde{\mathcal{B}}_{\text{QOED}}(\pi)
:= \tr\!\left(\bm{\mathcal{I}}^{\pi}_{\mathbf{k}\mid\overline{\mathbf{k}}}\right),
\end{equation*}
where $\bm{\mathcal{I}}^{\pi}_{\mathbf{k}\mid\overline{\mathbf{k}}}$ is the Schur complement of
$\bm{\mathcal{F}}^{\pi}$ w.r.t.\ the block $\overline{\mathbf{k}}$.
Assume $\bm{\mathcal{F}}^{\pi}_{\mathbf{k}\mathbf{k}}\succ 0$ and
$\bm{\mathcal{F}}^{\pi}_{\overline{\mathbf{k}}\overline{\mathbf{k}}}\succ 0$ for all $\pi\in\Pi$.
For each $\pi$, define
\begin{equation*}
\eta(\pi)=\frac{\tr(\bm{\mathcal{F}}^{\pi}_{\overline{\mathbf{k}}\overline{\mathbf{k}}})}
               {\tr(\bm{\mathcal{F}}^{\pi}_{\mathbf{k}\mathbf{k}})},
\quad
\beta(\pi)=
\left\|
\big(\bm{\mathcal{F}}^{\pi}_{\mathbf{k}\mathbf{k}}\big)^{-1/2}
\bm{\mathcal{F}}^{\pi}_{\mathbf{k}\overline{\mathbf{k}}}
\big(\bm{\mathcal{F}}^{\pi}_{\overline{\mathbf{k}}\overline{\mathbf{k}}}\big)^{-1/2}
\right\|_2^2.
\end{equation*}

If $\eta(\pi)\le \bar{\eta}<\infty$ and $\beta(\pi)\le \bar{\beta}<1$ for all $\pi\in\Pi$, then any maximizer
$\widehat{\pi}\in\arg\max_{\pi\in\Pi}\widetilde{\mathcal{B}}_{\text{QOED}}(\pi)$ satisfies
\begin{equation*}
\tr\!\big(\bm{\mathcal{F}}^{\widehat{\pi}}\big)
\;\ge\;
\frac{1-\bar{\beta}}{1+\bar{\eta}}\;
\max_{\pi\in\Pi}\tr\!\big(\bm{\mathcal{F}}^{\pi}\big).
\end{equation*}
\end{corollary}
\end{mdframed}

Moreover, we show that the agnostic objective is not quasi-optimal in general.

\begin{mdframed}[hidealllines=true,backgroundcolor=MaterialBlue10!40,innerleftmargin=.2cm,
  innerrightmargin=.2cm]
\begin{proposition}[The agnostic objective is not quasi-optimal in general]
\label{prop:naive-not-quasiopt}
Fix any nonempty index set $\mathbf{k}\subseteq\{1,\dots,m\}$ and define the agnostic objective
$\mathcal{B}_{\text{Agnostic}}(\bm{\pi}) := \tr(\bm{\mathcal{F}}^{\bm{\pi}}_{\mathbf{k}\mathbf{k}})$.
There does not exist a universal constant $\rho>0$ such that for all instances
(policy classes $\Pi$ and Fisher maps $\bm{\pi}\mapsto \bm{\mathcal{F}}^{\bm{\pi}}\succeq 0$),
every maximizer $\widehat{\bm{\pi}}\in\argmax_{\bm{\pi}}\mathcal{B}_{\text{Agnostic}}(\bm{\pi})$ satisfies
\[
\tr(\bm{\mathcal{F}}^{\widehat{\bm{\pi}}})
\;\ge\;
\rho\cdot \max_{\bm{\pi}\in\Pi}\tr(\bm{\mathcal{F}}^{\bm{\pi}}).
\]
In fact, for any $\rho\in(0,1)$ there exists an instance with two policies for which the ratio
$\tr(\bm{\mathcal{F}}^{\widehat{\bm{\pi}}})/\max_{\bm{\pi}}\tr(\bm{\mathcal{F}}^{\bm{\pi}})$ is smaller than $\rho$.
\end{proposition}
\end{mdframed}

\begin{proof}
It suffices to construct, for any target $\rho\in(0,1)$, an instance where the agnostic maximizer achieves less than a $\rho$-fraction of the full-trace optimum.

Consider $m=2$ parameters and $\mathbf{k}=\{1\}$.
Let the policy class contain exactly two policies, $\Pi=\{\bm{\pi}_A,\bm{\pi}_B\}$, and define the
corresponding Fisher information matrices by
\begin{equation*}
\bm{\mathcal{F}}^{\bm{\pi}_A}
=
\begin{bmatrix}
1+\delta & 0\\
0 & 0
\end{bmatrix},
\qquad
\bm{\mathcal{F}}^{\bm{\pi}_B}
=
\begin{bmatrix}
1 & 0\\
0 & M
\end{bmatrix},
\end{equation*}
where $\delta>0$ is fixed and $M>0$ will be chosen.

Both matrices are symmetric positive semidefinite and are valid Fisher matrices; for example, they can be realized
as covariance matrices of a score random vector $\mathbf{g}\sim \mathcal N(0,\bm{\mathcal{F}}^{\bm{\pi}})$.

First, the agnostic objective is $\mathcal{B}_{\text{Agnostic}}(\bm{\pi})=\tr(\bm{\mathcal{F}}^{\bm{\pi}}_{\mathbf{k}\mathbf{k}})
=(\bm{\mathcal{F}}^{\bm{\pi}})_{11}$.
Thus,
\begin{equation*}
\mathcal{B}_{\text{Agnostic}}(\bm{\pi}_A)=1+\delta
\;>\;
1=\mathcal{B}_{\text{Agnostic}}(\bm{\pi}_B),
\end{equation*}
so the unique agnostic maximizer is $\widehat{\bm{\pi}}=\bm{\pi}_A$.

Second, the full-trace objective is $\tr(\bm{\mathcal{F}}^{\bm{\pi}})$, so
\begin{equation*}
\tr(\bm{\mathcal{F}}^{\bm{\pi}_A})=1+\delta,
\qquad
\tr(\bm{\mathcal{F}}^{\bm{\pi}_B})=1+M.
\end{equation*}
For any $M>\delta$, the full BOED maximizer is $\bm{\pi}^\star=\bm{\pi}_B$.

Third, we show the arbitrarily small approximation ratio. The achieved ratio is
\begin{equation*}
\frac{\tr(\bm{\mathcal{F}}^{\widehat{\bm{\pi}}})}{\tr(\bm{\mathcal{F}}^{\bm{\pi}^\star})}
=
\frac{1+\delta}{1+M}.
\end{equation*}

Choose $M$ large enough so that $(1+\delta)/(1+M)<\rho$, e.g.
$M>\frac{1+\delta}{\rho}-1$.
Then the agnostic optimizer achieves less than a $\rho$-fraction of the full BOED optimum.
Since $\rho\in(0,1)$ was arbitrary, no universal constant-factor (quasi-optimality) guarantee is possible.
\end{proof}

\subsection{Eigen-subspace projection identity for the score}
\label{apdx:proof-score-subspace}

This appendix records an exact identity used by the \emph{eigen-subspace viewpoint} (active subspaces).
Let $\mathbf{g}=\nabla_{\bm{\phi}}\log q_{\bm{\theta}}(\bm{\tau}\vert \bm{\phi},\bm{\pi})$ be the trajectory score and
$\bm{\mathcal{F}}_{\bm{\phi}}=\mathbb{E}[\mathbf{g}\mathbf{g}^\top]\succeq \mathbf{0}$ the Fisher information matrix (FIM).
If we project $\mathbf{g}$ onto an \emph{eigenspace} of $\bm{\mathcal{F}}_{\bm{\phi}}$, then the expected squared norm of the discarded
(residual) score equals the \emph{omitted eigenvalue mass}.

To avoid confusion with the coordinate index set $\mathbf{k}$ in the main text, we use
$\mathbf{o}\subseteq\{1,\dots,m\}$ to denote an \emph{eigen-index set} in this appendix (i.e., it selects columns of $\mathbf{W}$).
If one instead truncates to a coordinate subset, the tail-eigenvalue identity generally does not hold; see Remark~\ref{rem:coord-truncation}.

\begin{mdframed}[hidealllines=true,backgroundcolor=MaterialBlue10!40,innerleftmargin=.2cm,
  innerrightmargin=.2cm]
\begin{proposition}[Eigen-subspace score truncation equals tail eigenvalue mass]
\label{prop:eig-subspace-score-trunc}
Fix $(\bm{\phi},\bm{\pi})$ and let $\bm{\mathcal{F}}_{\bm{\phi}}=\mathbf{W}\mathbf{\Lambda}\mathbf{W}^\top$ be an eigen-decomposition with
$\mathbf{\Lambda}=\mathrm{diag}(\lambda_1,\dots,\lambda_m)$ and $\lambda_1\ge\cdots\ge \lambda_m\ge 0$.
For any eigen-index set $\mathbf{o}\subseteq\{1,\dots,m\}$, let $\mathbf{W}_{\mathbf{o}}=\mathbf{W}[:,\mathbf{o}]$
and define the orthogonal projector $\mathbf{P}_{\mathbf{o}}=\mathbf{W}_{\mathbf{o}}\mathbf{W}_{\mathbf{o}}^\top$.
Then the discarded score energy satisfies the exact identity
\begin{equation*}
\label{eq:apdx-score-subspace-error}
\mathbb{E}\!\left[\left\|(\mathbf{I}-\mathbf{P}_{\mathbf{o}})\mathbf{g}\right\|_2^2\right]
=
\tr\!\left((\mathbf{I}-\mathbf{P}_{\mathbf{o}})\bm{\mathcal{F}}_{\bm{\phi}}\right)
=
\sum_{i\notin \mathbf{o}} \lambda_i.
\end{equation*}
\end{proposition}
\end{mdframed}

\begin{proof}
Let $\overline{\mathbf{o}}:=\{1,\dots,m\}\setminus \mathbf{o}$ and define
$\mathbf{W}_{\overline{\mathbf{o}}}:=\mathbf{W}[:,\overline{\mathbf{o}}]$ so that
$\mathbf{W}=[\mathbf{W}_{\mathbf{o}}\ \mathbf{W}_{\overline{\mathbf{o}}}]$ and
$\mathbf{I}-\mathbf{P}_{\mathbf{o}}=\mathbf{W}_{\overline{\mathbf{o}}}\mathbf{W}_{\overline{\mathbf{o}}}^\top$.

First, since $\mathbf{W}$ is orthonormal, define $\tilde{\mathbf{g}}:=\mathbf{W}^\top \mathbf{g}$ and partition
\begin{equation*}
\tilde{\mathbf{g}}=
\begin{bmatrix}
\tilde{\mathbf{g}}_{\mathbf{o}}\\
\tilde{\mathbf{g}}_{\overline{\mathbf{o}}}
\end{bmatrix}
=
\begin{bmatrix}
\mathbf{W}_{\mathbf{o}}^\top \mathbf{g}\\
\mathbf{W}_{\overline{\mathbf{o}}}^\top \mathbf{g}
\end{bmatrix}.
\end{equation*}

Second, we express the residual and its norm. Using $\mathbf{I}-\mathbf{P}_{\mathbf{o}}=\mathbf{W}_{\overline{\mathbf{o}}}\mathbf{W}_{\overline{\mathbf{o}}}^\top$,
\begin{equation*}
(\mathbf{I}-\mathbf{P}_{\mathbf{o}})\mathbf{g}
=
\mathbf{W}_{\overline{\mathbf{o}}}\mathbf{W}_{\overline{\mathbf{o}}}^\top \mathbf{g}
=
\mathbf{W}_{\overline{\mathbf{o}}}\tilde{\mathbf{g}}_{\overline{\mathbf{o}}}.
\end{equation*}

Because $\mathbf{W}_{\overline{\mathbf{o}}}$ has orthonormal columns,
$\| \mathbf{W}_{\overline{\mathbf{o}}}\mathbf{z}\|_2=\|\mathbf{z}\|_2$, hence
$\|(\mathbf{I}-\mathbf{P}_{\mathbf{o}})\mathbf{g}\|_2^2=\|\tilde{\mathbf{g}}_{\overline{\mathbf{o}}}\|_2^2$.
Third, we take expectations using the eigen-decomposition. Since $\bm{\mathcal{F}}_{\bm{\phi}}=\mathbb{E}[\mathbf{g}\mathbf{g}^\top]$,
\begin{equation*}
\mathbb{E}[\tilde{\mathbf{g}}\tilde{\mathbf{g}}^\top]
=
\mathbb{E}[\mathbf{W}^\top \mathbf{g}\mathbf{g}^\top \mathbf{W}]
=
\mathbf{W}^\top \bm{\mathcal{F}}_{\bm{\phi}} \mathbf{W}
=
\mathbf{\Lambda}.
\end{equation*}

Therefore $\mathbb{E}[\tilde{\mathbf{g}}_{\overline{\mathbf{o}}}\tilde{\mathbf{g}}_{\overline{\mathbf{o}}}^\top]
=\mathbf{\Lambda}_{\overline{\mathbf{o}}}=\mathrm{diag}(\{\lambda_i\}_{i\notin\mathbf{o}})$ and
\begin{equation*}
\mathbb{E}\!\left[\left\|\tilde{\mathbf{g}}_{\overline{\mathbf{o}}}\right\|_2^2\right]
=
\tr(\mathbf{\Lambda}_{\overline{\mathbf{o}}})
=
\sum_{i\notin\mathbf{o}}\lambda_i.
\end{equation*}
Combining with the second step gives the last equality. Finally, the trace form follows from
$\mathbb{E}\|(\mathbf{I}-\mathbf{P}_{\mathbf{o}})\mathbf{g}\|_2^2
=\mathbb{E}[\mathbf{g}^\top(\mathbf{I}-\mathbf{P}_{\mathbf{o}})\mathbf{g}]
=\tr((\mathbf{I}-\mathbf{P}_{\mathbf{o}})\mathbb{E}[\mathbf{g}\mathbf{g}^\top])$,
which yields the result.
\end{proof}

\begin{mdframed}[hidealllines=true,backgroundcolor=MaterialBlueGray10!40,innerleftmargin=.2cm,
  innerrightmargin=.2cm]
\begin{remark}[Coordinate truncation is different]
\label{rem:coord-truncation}
If one truncates to a coordinate subset $\mathbf{k}\subseteq\{1,\dots,m\}$, the projector is instead
$\mathbf{P}_{\mathbf{k}}:=\mathbf{S}_{\mathbf{k}}^\top\mathbf{S}_{\mathbf{k}}$ where $\mathbf{S}_{\mathbf{k}}$ selects coordinates
(so $\mathbf{S}_{\mathbf{k}}\mathbf{g}=\mathbf{g}_{\mathbf{k}}$). In this case
\begin{equation*}
\mathbb{E} \left[\|(\mathbf{I}-\mathbf{P}_{\mathbf{k}})\mathbf{g}\|_2^2\right]
=
\tr \left((\mathbf{I}-\mathbf{P}_{\mathbf{k}})\bm{\mathcal{F}}_{\bm{\phi}} \right)
=
\tr(\bm{\mathcal{F}}_{\bm{\phi}})-\tr(\bm{\mathcal{F}}_{\mathbf{k}\mathbf{k}}),
\end{equation*}
which generally does \emph{not} equal a tail eigenvalue sum unless the leading eigenspace is (approximately) aligned with
the coordinate axes. Moreover, by the Ky Fan maximum principle, for any rank-$r$ orthogonal projector $\mathbf{P}$,
$\tr(\mathbf{P}\bm{\mathcal{F}})\le \sum_{i=1}^{r}\lambda_i$.
Since $\mathbf{P}_{\mathbf{k}}$ has rank $|\mathbf{k}|$, this implies
$\tr((\mathbf{I}-\mathbf{P}_{\mathbf{k}})\bm{\mathcal{F}})\ge \sum_{i=|\mathbf{k}|+1}^m \lambda_i$.
\end{remark}
\end{mdframed}

For completeness, we also include a standard subspace-approximation bound from the active-subspaces literature~\cite{ActiveSubspaces}.

\begin{mdframed}[hidealllines=true,backgroundcolor=MaterialBlue10!40,innerleftmargin=.2cm,
  innerrightmargin=.2cm]
\begin{theorem}[Subspace approximation bound]
\label{thm:subspace-appro-bnd}
Let $f:\mathbb{R}^m\to\mathbb{R}$ be continuously differentiable, and let $\bm{\rho}$ be a probability density on $\mathbb{R}^m$
that satisfies a Poincar\'e inequality with constant $\delta_{\bm{\rho}}>0$.
Define
\begin{equation*}
    \bm{\mathcal{F}}
    :=
    \mathbb{E}_{\bm{\phi}\sim\bm{\rho}}
    \big[
        \nabla f(\bm{\phi}) \nabla f(\bm{\phi})^\top
    \big],
\end{equation*}
and let $\bm{\mathcal{F}}=\mathbf{W}\mathbf{\Lambda} \mathbf{W}^\top$ be its eigen-decomposition with sorted
$\mathbf{\Lambda}=\mathrm{diag}(\lambda_1,\dots,\lambda_m)$. Partition $\mathbf{W}=[\mathbf{W}_1\; \mathbf{W}_2]$ where
$\mathbf{W}_1\in\mathbb{R}^{m\times n}$ contains the first $n$ eigenvectors.
Define the conditional expectation ridge approximation
\begin{equation*}
\bar{f}(\mathbf{y}) := \mathbb{E}_{\bm{\phi}\sim\bm{\rho}}\!\left[f(\bm{\phi}) \,\middle|\, \mathbf{W}_1^\top \bm{\phi}=\mathbf{y}\right].
\end{equation*}
Then
\begin{equation*}
    \mathbb{E}_{\bm{\phi}\sim\bm{\rho}}
    \Big[
        \big(
            f(\bm{\phi})
            -
            \bar{f}(\mathbf{W}_1^\top \bm{\phi})
        \big)^2
    \Big]
    \le
    \delta_{\bm{\rho}}
    \sum_{i=n+1}^m \lambda_i.
\end{equation*}
\end{theorem}
\end{mdframed}

\begin{proof}
Let $\bm{\mathcal{F}}=\mathbb{E}_{\bm{\phi}\sim\bm{\rho}}\!\left[\nabla f(\bm{\phi})\,\nabla f(\bm{\phi})^{\top}\right]$
and let $\bm{\mathcal{F}}=\mathbf{W}\mathbf{\Lambda} \mathbf{W}^{\top}$ with $\mathbf{W}=[\mathbf{W}_1\;\mathbf{W}_2]$, where
$\mathbf{W}_1\in\mathbb{R}^{m\times n}$ contains the first $n$ eigenvectors and $\mathbf{W}_2\in\mathbb{R}^{m\times(m-n)}$ contains the remaining eigenvectors.
Define the orthogonal projections
\begin{equation*}
    \mathbf{P}_1 = \mathbf{W}_1 \mathbf{W}_1^{\top},
    \qquad
    \mathbf{P}_2 = \mathbf{W}_2 \mathbf{W}_2^{\top} = \mathbf{I} - \mathbf{P}_1.
\end{equation*}

First, we work in rotated coordinates. Because $\mathbf{W}$ is orthonormal, every $\bm{\phi}\in\mathbb{R}^m$ can be written as
\begin{equation*}
    \bm{\phi} = \mathbf{W}_1 \mathbf{y} + \mathbf{W}_2 \mathbf{z},
    \qquad
    \mathbf{y}=\mathbf{W}_1^{\top}\bm{\phi},\;\;
    \mathbf{z}=\mathbf{W}_2^{\top}\bm{\phi}.
\end{equation*}

Second, we use the best approximation that depends only on $\mathbf{W}_1^{\top}\bm{\phi}$, namely $\bar{f}(\mathbf{y})$ defined above.
By the law of total variance,
\begin{equation*}
\label{eq:apdx-total-variance}
\mathbb{E}_{\bm{\phi}\sim\bm{\rho}}\Big[\big(f(\bm{\phi})-\bar{f}(\mathbf{W}_1^{\top}\bm{\phi})\big)^2\Big]
=
\mathbb{E}_{\mathbf{y}}\Big[
        \mathrm{Var}\big(f(\bm{\phi}) \,\vert\, \mathbf{W}_1^{\top}\bm{\phi}=\mathbf{y}\big)
\Big].
\end{equation*}

Third, assume the conditional distributions of $\mathbf{z}=\mathbf{W}_2^\top\bm{\phi}$ given $\mathbf{y}=\mathbf{W}_1^\top\bm{\phi}$ satisfy a Poincar\'e
inequality with constant $\delta_{\bm{\rho}}$ (as in standard active-subspaces analyses~\cite{ActiveSubspaces}). Fix $\mathbf{y}$ and define the function of $\mathbf{z}$:
\begin{equation*}
    h_{\mathbf{y}}(\mathbf{z})
    =
    f(\mathbf{W}_1\mathbf{y} + \mathbf{W}_2\mathbf{z}).
\end{equation*}

Applying the Poincar\'e inequality (with constant $\delta_{\bm{\rho}}$) gives
\begin{equation*}
\label{eq:apdx-poincare-conditional}
\mathrm{Var}\big(f(\bm{\phi}) \,\vert\, \mathbf{W}_1^{\top}\bm{\phi}=\mathbf{y}\big)
\le
\delta_{\bm{\rho}}\;
\mathbb{E}\Big[\big\|\nabla_{\mathbf{z}} h_{\mathbf{y}}(\mathbf{z})\big\|_2^2
    \,\Big|\,
    \mathbf{W}_1^{\top}\bm{\phi}=\mathbf{y}\Big].
\end{equation*}

By the chain rule,
$
\nabla_{\mathbf{z}} h_{\mathbf{y}}(\mathbf{z})
=
\mathbf{W}_2^{\top}\nabla_{\bm{\phi}} f(\bm{\phi}),
$
so taking expectation over $\mathbf{y}$ yields
\begin{equation*}
\label{eq:apdx-mse-to-projgrad}
\mathbb{E}_{\bm{\phi}\sim\bm{\rho}}\Big[\big(f(\bm{\phi})-\bar{f}(\mathbf{W}_1^{\top}\bm{\phi})\big)^2\Big]
\le
\delta_{\bm{\rho}}\;
\mathbb{E}_{\bm{\phi}\sim\bm{\rho}}\Big[\big\|\mathbf{W}_2^{\top}\nabla f(\bm{\phi})\big\|_2^2\Big].
\end{equation*}

Finally, rewrite the projected gradient term using eigenvalues:
\begin{equation*}
\mathbb{E}\!\left[\|\mathbf{W}_2^{\top}\nabla f(\bm{\phi})\|_2^2\right]
=
\tr\!\left(\mathbf{W}_2^{\top}\bm{\mathcal{F}}\mathbf{W}_2\right)
=
\sum_{i=n+1}^{m}\lambda_i.
\end{equation*}

Combining with the previous equation proves the bound.
\end{proof}

\subsection{Policy Parametrization}
\label{apdx:model-struc}

\textbf{Shortcut Model.}
Training uses the AdamW optimizer with a learning rate of $\num{3e-4}$, and we clip the gradient norm at $\num{5.0}$. The replay buffer has a maximum capacity of $\num{1e6}$ transitions and a batch size of $\num{1024}$. The detailed architecture is shown below, where \texttt{STATE}, \texttt{ACTION}, and \texttt{PARAMETER} denote the dimensions of the state, action, and physical coefficients.

\begin{codesnippet}
Shortcut Model: ShortcutModel(
  (flow): DiffusionTransformer(
    (state_encoder): Sequential(
      (0): Linear(in_features=STATE, out_features=512)
      (1): Mish()
      (2): Linear(in_features=512, out_features=128)
    )
    (state_decoder): Sequential(
      (0): Linear(in_features=128, out_features=512)
      (1): Mish()
      (2): Linear(in_features=512, out_features=STATE)
    )
    (action_encoder): Sequential(
      (0): Linear(in_features=ACTION, out_features=512)
      (1): Mish()
      (2): Linear(in_features=512, out_features=128)
    )
    (privilege_encoder): Sequential(
      (0): Linear(in_features=PARAMETER, out_features=512)
      (1): Mish()
      (2): Linear(in_features=512, out_features=128)
    )
    (timestep_embedding): DualTimestepEncoder(
      (sinusoidal_pos_emb): SinusoidalPosEmb()
      (proj): Sequential(
        (0): Linear(in_features=128, out_features=256)
        (1): Mish()
        (2): Linear(in_features=256, out_features=64)
      )
    )
    (d_timestep_embedding): DualTimestepEncoder(
      (sinusoidal_pos_emb): SinusoidalPosEmb()
      (proj): Sequential(
        (0): Linear(in_features=128, out_features=256)
        (1): Mish()
        (2): Linear(in_features=256, out_features=64)
      )
    )
    (transformer): ModuleList(
      (0-5): 6 x AdaLNAttentionBlock(
        (norm1): LayerNorm((128,), eps=1e-06)
        (attn): Attention(
          (qkv): Linear(in_features=128, out_features=384)
          (attn_drop): Dropout(p=0.0, inplace=False)
          (proj): Linear(in_features=128, out_features=128)
          (proj_drop): Dropout(p=0.0, inplace=False)
        )
        (norm2): LayerNorm((128,), eps=1e-06)
        (mlp): MLP(
          (fc1): Linear(in_features=128, out_features=512)
          (act): GELU(approximate=none)
          (fc2): Linear(in_features=512, out_features=128)
          (drop): Dropout(p=0.0, inplace=False)
        )
        (adaLN_modulation): Sequential(
          (0): SiLU()
          (1): Linear(in_features=512, out_features=768)
        )
      )
      (6): AdaLNFinalLayer(
        (norm): LayerNorm((128,), eps=1e-06)
        (linear): Linear(in_features=128, out_features=128)
        (adaLN_modulation): Sequential(
          (0): SiLU()
          (1): Linear(in_features=512, out_features=256)
        )
      )
    )
  )
)
\end{codesnippet}

\begin{table*}[t!]
\centering
\small
\caption{Domain Randomization Parameters across all platforms. Ranges denote the Uniform distribution $\mathcal{U}(min, max)$. Operations indicate whether the noise replaces ($=$), scales ($\times$), or is added to ($+$) the nominal values. ``Obj'' refers to the manipulated object for the Hand task.}
\label{tab:dr_params_final}
\begin{tabular}{l c c c c c}
    \toprule
    \textbf{Parameter} & \textbf{Op.} & \textbf{Go1 Quad} & \textbf{G1 Humanoid} & \textbf{Jackal Vehicle} & \textbf{Inspire Hand} \\
    \midrule
    Environment Friction & $=$ & $\mathcal{U}(0.4, 1.0)$ & $\mathcal{U}(0.4, 1.0)$ & $\mathcal{U}(0.1, 1.0)$ & $\mathcal{U}(0.1, 1.0)$ \\
    Joint Friction Loss & $\times$ & $\mathcal{U}(0.1, 2.0)$ & $\mathcal{U}(0.1, 2.0)$ & $\mathcal{U}(0.1, 2.0)$ & $\mathcal{U}(0.1, 2.0)$ \\
    Joint Armature & $\times$ & $\mathcal{U}(0.1, 5.0)$ & $\mathcal{U}(0.1, 5.0)$ & $\mathcal{U}(0.8, 1.2)$ & $\mathcal{U}(1.0, 1.05)$ \\
    Joint Damping & $\times$ & -- & -- & $\mathcal{U}(0.5, 2.0)$ & $\mathcal{U}(0.1, 1.2)$ \\
    Joint Stiffness & $\times$ & -- & -- & -- & $\mathcal{U}(0.1, 3.0)$ \\
    Link Masses & $\times$ & $\mathcal{U}(0.1, 5.0)$ & $\mathcal{U}(0.1, 5.0)$ & $\mathcal{U}(0.1, 3.0)$ & $\mathcal{U}(0.1, 3.0)$ \\
    Base/Torso Mass & $+$ & $\mathcal{U}(-1.0, 10.0)$ & $\mathcal{U}(0.0, 10.0)$ & -- & -- \\
    Base/Obj Inertia & $\times$ & -- & -- & $\mathcal{U}(0.8, 1.2)$ & $\mathcal{U}(0.8, 1.2)$ \\
    CoM / Obj Position & $+$ & $\mathcal{U}(-0.1, 0.1)$ & -- & -- & $\mathcal{U}(-0.005, 0.005)$ \\
    Initial Pos ($q_0$) & $+$ & $\mathcal{U}(-0.1, 0.1)$ & $\mathcal{U}(-0.1, 0.1)$ & $\mathcal{U}(-0.05, 0.05)$ & $\mathcal{U}(-0.05, 0.05)$ \\
    \bottomrule
\end{tabular}
\end{table*}

\textbf{PPO and Baselines.}
Following the SAC tuning guide in massive simulations\footnote{\url{https://araffin.github.io/post/sac-massive-sim/}}, we highlight the specific parameters that differ from the stable-baselines3 default implementation.

\begin{table}[htbp]
\centering
\small
\begin{tabular}{lc}
\noalign{\hrule height 1pt}
\textbf{Hyperparameter}                     & \textbf{Value}      \\ \noalign{\hrule height 1pt}

\textbf{PPO}~\cite{schulman2017proximal} & \\
Batch Size                  & \num{256}    \\
Entropy Cost               & \num{1e-2}  \\
Max Grad Norm              & \num{1} \\
Num Minibatch             & \num{32} \\
Num Updates per Batch     & \num{4} \\
Discount                   & \num{0.97} \\
Learning Rate             & \num{3e-4} \\

\hline
\textbf{SAC}~\cite{haarnoja2018soft} & \\
Learning Starts                  & \num{1e5}    \\
Entropy Coefficient            & Auto \num{0.06} \\
Gradient Steps           & \num{64} \\
Batch Size                   & \num{256}    \\
Gamma             & \num{0.99}    \\
Learning Rate                   & \num{1e-4}    \\
Tau             & \num{0.05}    \\
Buffer Size               & \num{2e6} \\

\hline
\textbf{S{\footnotesize AC}-A{\footnotesize DAPT}}~\cite{chen22eipo} & \\
Beta Lower Bound                   & \num{-3.65}    \\
Beta Upper Bound            & \num{4.66} \\
Shift Multiplier                  & \num{6.86}       \\

\hline
\textbf{Shared Configuration} & \\
Policy                   & $ [512, 256, 128] $    \\
Value            & $ [512, 256, 128] $ \\

\noalign{\hrule height 1pt}
\end{tabular}
\caption{Hyperparameters for policies and the model-based policy optimization.}
\label{tab:tdmpc2-hyperparameters}
\end{table}

\subsection{Experiments}
\label{apdx:exp}

\begin{figure}[t!]
\centering
\includegraphics[width=0.485\textwidth]{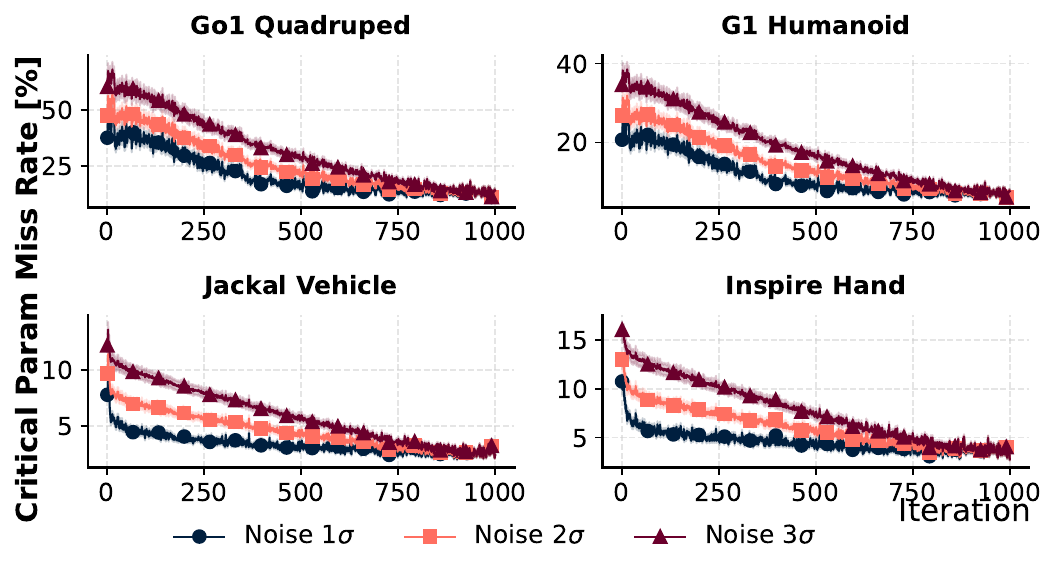}
\vspace*{-0.3in}
\caption{\small Critical parameter identification miss rate via our learned dynamics with respect to learning iterations.}
\vspace*{-0.2in}
\label{fig:param-miss}
\end{figure}

\textbf{Domain Randomization.} 
Table~\ref{tab:dr_params_final} details the domain randomization parameters used for each robot platform. We apply uniform noise $\mathcal{U}(a, b)$ to physical properties including friction, mass, and actuator dynamics. The randomization strategy respects the physical nature of each parameter: strictly positive quantities (e.g., friction, armature, damping) are perturbed via multiplicative scaling ($\times$), while state offsets and payloads are perturbed additively ($+$). We adopt well-studied randomization ranges from mjlab~\cite{mjlab} to ensure physically plausible  model mismatches. Notably, the Go1 quadruped, G1 humanoid, and Inspire Hand use the default mjlab reward functions, while the Jackal vehicle follows the reward in \cite{yu2024adaptive}.

\textbf{Learned Dynamics Model.} 
Fig.~\ref{fig:param-miss} shows the critical parameter miss rate, averaged across Flat and Rough where applicable. As robot degrees of freedom increase, the learned dynamics initially exhibits higher miss rates. However, performance improves significantly as training progresses, stabilizing around \num{750} iterations (i.e., timesteps). 

\textbf{Real-World Wheeled Navigation.} 
In the real-world wheeled-navigation experiment, D{\footnotesize ISAGREEMENT} drops substantially relative to simulation, mainly due to the runtime cost of maintaining and evaluating an ensemble of five dynamics models. We attempted to reduce this overhead with \texttt{torch.compile}, but further work is needed to make ensemble-based exploration practical under onboard compute constraints. Q{\footnotesize OED}-P{\footnotesize HYSICS} primarily fails because real-time sim mirroring is difficult: mapping noisy observations into the simulator introduces additional error, and simulator \texttt{reset} is a major computational bottleneck. While prior work can mitigate this with extensive domain randomization over physics and environment (e.g., terrain elevation), such pipelines typically require offboard compute and are difficult to run fully onboard on a Jetson Orin SoC. S{\footnotesize AC}-A{\footnotesize DAPT} often exhibits wobbling behaviors as exploration, but these motions are less informative than physics-targeted probing. In contrast, physics-driven exploration can induce task-relevant behaviors (e.g., lifting to probe mass and scrubbing to probe friction), as illustrated in our manipulation experiment.

\textbf{Validation of Theorem~\ref{thm:qoed-quasiopt-trace}.}
Using recorded data, we empirically measured the theorem quantities $(\eta, \beta)$ and the induced factor $\rho=\nicefrac{(1-\beta)}{(1+\eta)}$. Since Theorem~\ref{thm:qoed-quasiopt-trace} is stated over the policy class $\Pi$, we report these values as an empirical calibration of the bound in our evaluated settings. Larger $\rho$ is better, and $\rho=1$ is ideal. In simulation, the tuples $(\eta,\beta,\rho)$ are: Go1 $(0.0012, 0.2784, 0.7207)$, G1 $(0.0166, 0.2731, 0.7150)$, Jackal $(0.0011, 0.0008, 0.9981)$, and Hand $(0.0009, 0.0008, 0.9983)$. On real robots, Franka gives $(0.0162, 0.1421, 0.8442)$ and Jackal gives $(0.0147, 0.0353, 0.9507)$. These values show that the bound is near-ideal in Jackal/Hand, strong on Franka and real Jackal, and still substantial in the harder Go1/G1 settings, so it is informative rather than vacuous.

\end{document}